\documentclass[conference]{IEEEtran}
\IEEEoverridecommandlockouts
\usepackage{cite}
\usepackage{amsmath,amssymb,amsfonts}
\usepackage{algorithm}
\usepackage{graphicx}
\usepackage{textcomp}
\usepackage{xcolor}
\def\BibTeX{{\rm B\kern-.05em{\sc i\kern-.025em b}\kern-.08em
    T\kern-.1667em\lower.7ex\hbox{E}\kern-.125emX}}

\usepackage{url}            % simple URL typesetting
\usepackage{amsfonts}       % blackboard math symbols
\usepackage{microtype}
\graphicspath{{figures/}}
\usepackage{subcaption}
\usepackage{booktabs} % for professional tables
\usepackage{colortbl}
\usepackage{mathtools}
\usepackage{multirow}
\usepackage{tabularx}
\usepackage{float}
\usepackage{algpseudocode}

\usepackage[inline]{enumitem}
\usepackage[english]{babel}
\usepackage{color,soul}
\usepackage{changepage}
\usepackage{xcolor}
\usepackage{amssymb}
\usepackage{amsmath}
\usepackage{amssymb}
\usepackage{mathtools}
\usepackage{amsthm}

\usepackage[capitalize,noabbrev]{cleveref}

\theoremstyle{plain}
\newtheorem{theorem}{Theorem}[section]

\theoremstyle{definition}

\newtheorem{assumption}[theorem]{Assumption}
\theoremstyle{remark}

\usepackage[textsize=tiny]{todonotes}

\begin{document}

\title{Staleness-Alleviated Distributed GNN Training via Online Dynamic-Embedding Prediction}

\author{\IEEEauthorblockN{Guangji Bai}
\IEEEauthorblockA{\textit{Department of Computer Science} \\
\textit{Emory University}\\
Atlanta, USA \\
guangji.bai@emory.edu}
\and
\IEEEauthorblockN{Ziyang Yu}
\IEEEauthorblockA{\textit{Department of Systems Design Engineering} \\
\textit{University of Waterloo}\\
Waterloo, Canada \\
z333yu@uwaterloo.ca}
\and
\IEEEauthorblockN{Zheng Chai}
\IEEEauthorblockA{\textit{Department of Computer Science} \\
\textit{University of Virginia}\\
Charlottesville, USA \\
dub6yh@virginia.edu}
\and
\IEEEauthorblockN{Yue Cheng}
\IEEEauthorblockA{\textit{Department of Computer Science} \\
\textit{University of Virginia}\\
Charlottesville, USA \\
mrz7dp@virginia.edu}
\and
\IEEEauthorblockN{Liang Zhao}
\IEEEauthorblockA{\textit{Department of Computer Science} \\
\textit{Emory University}\\
Atlanta, USA \\
liang.zhao@emory.edu}
}

\maketitle

\begin{abstract}
Despite the recent success of Graph Neural Networks (GNNs), it remains challenging to train GNNs on large-scale graphs due to neighbor explosions. As a remedy, distributed computing becomes a promising solution by leveraging abundant computing resources (e.g., GPU). However, the node dependency of graph data increases the difficulty of achieving high concurrency in distributed GNN training, which suffers from the massive communication overhead. To address it, \emph{Historical value approximation} is deemed a promising class of distributed training techniques. It utilizes an offline memory to cache historical information (e.g., node embedding) as an affordable approximation of the exact value and achieves high concurrency. However, such benefits come at the cost of involving dated training information, leading to staleness, imprecision, and convergence issues. To overcome these challenges, this paper proposes SAT (Staleness-Alleviated Training), a novel and scalable distributed GNN training framework that reduces the embedding staleness adaptively. The key idea of SAT is to model the GNN's embedding evolution as a temporal graph and build a model upon it to predict future embedding, which effectively alleviates the staleness of the cached historical embedding. We propose an online algorithm to train the embedding predictor and the distributed GNN alternatively and further provide a convergence analysis. Empirically, we demonstrate that SAT can effectively reduce embedding staleness and thus achieve better performance and convergence speed on multiple large-scale graph datasets.
\end{abstract}


\section{Introduction}
\label{sec:introduction}

\emph{Graph Neural Networks} (GNNs) have shown impressive success in analyzing non-Euclidean graph data and have achieved promising results in various applications, including social networks, recommender systems, and knowledge graphs, etc.~\cite{dai2016discriminative,ying2018graph,lei2019gcn}. Despite their great promise, GNNs meet significant challenges when being applied to large graphs, which are common in the real world---the number of nodes goes beyond millions or even billions. 
Training GNNs on large graphs is jointly challenged by the lack of inherent parallelism in the backpropagation optimization and heavy inter-dependencies among graph nodes, rendering existing parallel techniques inefficient. To tackle such unique challenges, distributed GNN training is a promising open domain that has attracted fast-increasing attention in recent years and has become the \emph{de facto} standard for fast and accurate training over large graphs~\cite{dorylus_osdi21,ramezani2021learn,wan2022pipegcn,chai2022distributed}.

A key challenge in distributed GNN training lies in obtaining accurate node embeddings based on the neighbor nodes and subgraphs while avoiding massive communication overhead incurred by the message passing across them. On the one hand, naively partitioning the graph into different subgraphs by dropping the edges across them can reduce communications among subgraphs. However, this will result in severe information loss and highly inaccurate approximation of node embeddings~\cite{angerd2020distributed,jia2020improving,ramezani2021learn}. On the other hand, propagating all the information between different subgraphs will guarantee accurate node embeddings, while inevitably suffering huge communication overhead and plagued efficiency due to neighbor explosion~\cite{ma2019neugraph,zhu2019aligraph,zheng2020distdgl,tripathy2020reducing}. More recently, using historical value to approximate the exact one has been widely used and achieved SOTA performance in large-scale GNN training~\cite{fey2021gnnautoscale,wan2022pipegcn,yu2022graphfm,chai2022distributed}. Specifically, by leveraging an offline memory to cache \emph{historical embeddings} (e.g., of the nodes) to approximate true embeddings, such methods can achieve a constant communication cost over graph size while the inter-dependency between subgraphs is retained. 

However, the aforementioned idea is bottlenecked by the \emph{staleness} of the historical embeddings. Such dated embeddings further lead to staleness and imprecision in the gradients of the embeddings and model parameters during the backward pass. As shown in Figure~\ref{fig: embedding staleness curve}, we measure the staleness of historical embeddings (red curves) of a GCN trained on 2 graph datasets, and the staleness error is nontrivial throughout the entire training. The staleness error degrades the model's performance and slows the convergence, which we empirically validated in our experiment section.

\begin{figure}[t!]
% \vspace{-0.2cm}
\centering
% \vspace{-0.3cm}
\begin{subfigure}{0.23\textwidth}
  \includegraphics[width=1\linewidth]{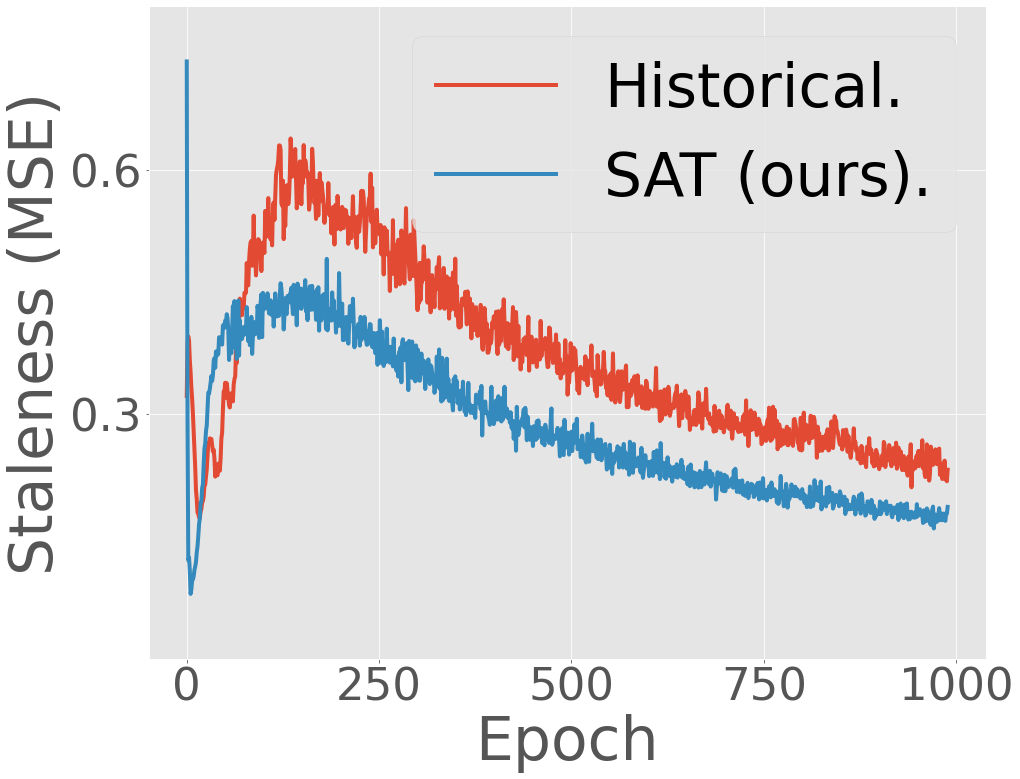}
  \caption{GCN on Flickr}
  % \label{fig:sensitivity moons} 
\end{subfigure}
% \vspace{-0.2cm}
\begin{subfigure}{0.23\textwidth}
  \includegraphics[width=1\linewidth]{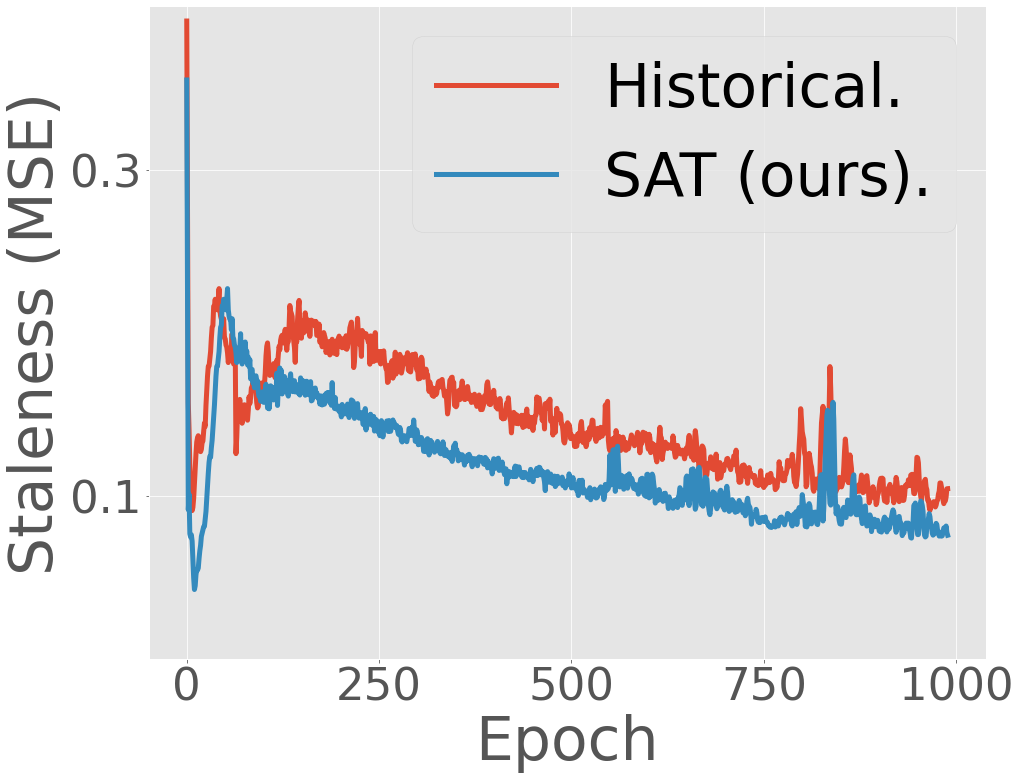}
  \caption{GCN on Reddit}
  % \label{fig:sensitivity elec2}
\end{subfigure}
% \vspace{-0.1cm}
\caption{\textbf{Embedding staleness and its alleviation.} Comparison of embedding staleness with (blue) or without (red) our method. The staleness error is measured concerning the full-graph training's embedding which has zero staleness.\vspace{-0.3cm}}
  \label{fig: embedding staleness curve}
\vspace{-2mm}
\end{figure}

Alleviating the staleness means making the historical embedding a more accurate and timely estimate of the actual embedding, which is, though appealing, difficult to achieve due to several challenges: \textbf{1). Difficulties in tracking the dynamics of the true embeddings.} Alleviating the staleness requires not only modeling the temporal patterns and trends of actual and historical embeddings across iterations but also capturing these embeddings' mutual dependencies used in GNN computation; \textbf{2). Difficulty in designing an efficient and scalable algorithm for staleness reduction.} Adjusting the historical embedding incurs extra computational overhead. So it is challenging to ensure a better trade-off between its cost and quality toward a substantial performance gain. \textbf{3). Unknown impact on the model training.} Using the adjusted historical embedding may change the properties of the training process of GNNs and trouble its convergence and stability. Theoretical analysis and guarantee are imperative yet nontrivial due to the involvement of historical adjustment.

\noindent\textbf{Our  Contribution.} To jointly address all challenges above, we propose a novel distributed GNN training framework toward an appealing trade-off between concurrency and quality of embedding calculation with \underline{\textbf{S}}taleness \underline{\textbf{A}}lleviation \underline{\textbf{T}}raining, or \textbf{SAT}.
In SAT, we design a new architecture called the \emph{embedding predictor} that handles the staleness issue in a data-driven manner and enjoys good scalability. We innovatively formulate the distributed GNN's embeddings as a sequence of \emph{temporal graphs} with their nodes \& edges induced by the original graph and the time defined as training epochs, where each temporal graph fully characterizes the evolution of node embeddings for each local GNN. Based on the temporal graphs, we further propose a multi-task learning loss to jointly optimize the embedding predictor and the temporal graphs, where each task corresponds to the embedding prediction on a specific local subgraph. In terms of the optimization of SAT framework, due to the fact that the parameters of the distributed GNN and the embedding predictor are \emph{coupled} and form a nested optimization problem, we propose an online algorithm to train each model alternatively and provide the theoretical guarantees of how the embedding predictor could affect the convergence under the distributed setting.
Finally, we perform extensive evaluations over 6 comparison methods on 8 real-world graph benchmarks with 2 different GNN operators, where our framework can boost existing state-of-the-art methods' performance and convergence speed by a great margin as a result of reduced staleness in node embeddings.

\section{Related Work}
\label{sec: related work}

\noindent\textbf{Distributed GNN Training.} 
The process of distributed GNN training necessitates the division of the original graph into multiple subgraphs, each of which is processed in parallel. The methodologies employed in such systems fall into two primary categories: \emph{partition-based} and \emph{propagation-based} training strategies.

Partition-based techniques segment graphs into subgraphs, allowing parallel training with reduced inter-subgraph communication but at the cost of significant information loss due to neglected node dependencies. In partition-based training, systems such as NeuGraph~\cite{ma2019neugraph} and AliGraph~\cite{yang2019aligraph} address GPU memory constraints by shuttling data partitions between GPU and storage. This swapping, however, is not without cost, adding significant overhead to the training process. LLCG~\cite{ramezani2021learn} proposes a decentralized training paradigm where subgraphs are processed independently, utilizing a central server for model aggregation. To counteract information loss from graph partitioning, LLCG employs a sampling strategy that attempts to preserve the global graph structure within each subgraph. Despite this, the approach struggles to fully encapsulate the global context, often at the expense of model performance.

In contrast, propagation-based methods maintain edge connections across subgraphs for neighbor aggregation, preserving information at the expense of increased communication overhead and potential training inefficiencies due to 'neighborhood explosion' as GNN depth increases. Propagation-based systems, exemplified by DGL~\cite{wang2019deep}, operate by sharing node representations across partitions. Unlike partition-based counterparts, DGL's strategy requires constant communication to exchange these representations during local training iterations, leading to substantial communication overhead. P3~\cite{gandhi2021p3} seeks to alleviate this overhead by partitioning both the feature and GNN layers, aiming to refine the model's internal information flow. However, P3's methodology imposes limitations on the dimensions of the GNN's hidden layers, which must be smaller than the input feature dimensions. This restriction has the potential to dilute the model's expressiveness and, consequently, its overall performance capabilities. For a more comprehensive literature review, one may refer to this survey~\cite{shao2022distributed}.

\noindent\textbf{GNNs with Historical Value Approximation.}
The idea of considering \emph{historical values} (of embedding, gradient, etc) as an approximation of the exact values can date back to distributed training on \emph{i.i.d} data~\cite{huo2018decoupled,xu2020acceleration}. With the rising popularity of GNNs, such an idea has been extended to train GNNs, especially on large-scale graphs. For example, in sampling-based methods, VR-GCN~\cite{chen2018stochastic} uses historical embeddings to reduce neighbor sampling variance. GNNAutoScale~\cite{fey2021gnnautoscale} leverages historical embeddings of \emph{1-hop} neighbors to achieve efficient mini-batch training. GraphFM~\cite{yu2022graphfm} applies a momentum step on historical embeddings to obtain better embedding approximations of sampled nodes. In distributed GNN training, PipeGCN~\cite{wan2022pipegcn} proposed a pipeline parallelism training for GNNs based on historical embeddings and gradients. DIGEST~\cite{chai2022distributed} leverages historical embeddings to achieve computation-storage separation and partition-parallel training.

\section{Preliminaries}

\noindent\textbf{Graph Neural Networks.} GNNs aim to learn a function of signals/features on a graph $\mathcal{G(V, E)}$ with node embeddings $\mathbf{X}\in\mathbb{R}^{\vert\mathcal{V}\vert\times d}$, where $d$ denotes the node feature dimension. For typical semi-supervised node classification tasks~\cite{kipf2016semi}, where each node $v\in\mathcal{V}$ is associated with a label $\mathbf{y}_v$, a $L$-layer GNN $\boldsymbol{f}_{\boldsymbol{\theta}}$ parameterized by $\boldsymbol{\theta}$ is trained to learn the node embedding $\mathbf{h}_v$ such that $\mathbf{y}_v$ can be predicted accurately. Analytically, the $\ell$-th layer of the GNN is defined as:
\begin{equation}
\label{eq:gnn}
% \vspace{-1mm}
\begin{split}
    \mathbf{h}_{v}^{(\ell+1)} &= \boldsymbol{f}_{\boldsymbol{\theta}}^{(\ell+1)} \Big(\mathbf{h}_{v}^{(\ell)}, \big\{\!\!\big\{ \mathbf{h}^{(\ell)}_u  \big\}\!\!\big\}_{u \in \mathcal{N}(v)} \Big) \\
    &= \Psi^{(\ell+1)}_{\boldsymbol{\theta}}\Big( \mathbf{h}_{v}^{(\ell)},  \Phi^{(\ell+1)}_{\boldsymbol{\theta}}\big(\big\{\!\!\big\{\mathbf{h}^{(\ell)}_u  \big\}\!\!\big\}_{u \in \mathcal{N}(v)}\big) \Big),
\end{split}
% \vspace{-1mm}
\end{equation}
where $\mathbf{h}_{v}^{(\ell)}$ denotes the embedding of node $v$ in the $\ell$-th layer, and $\mathbf{h}^{(0)}_{v}$ being initialized to $\mathbf{x}_v$ ($v$-th row in $\mathbf{X}$), and $\mathcal{N}(v)$ represents the set of neighborhoods for node $v$. Each layer of the GNN, i.e.  $\boldsymbol{f}_{\boldsymbol{\theta}}^{(\ell)}$, can be further decomposed into the aggregation function $\Phi^{(\ell)}_{\boldsymbol{\theta}}$ and the updating function $\Psi^{(\ell)}_{\boldsymbol{\theta}}$, and both functions can choose to use various functions in different types of GNNs.

\noindent\textbf{Distributed Training for GNNs.} Distributed GNN training first partitions the original graph into multiple subgraphs without overlap, which can also be considered mini-batches. Then different mini-batches are trained in different devices in parallel. Here, Eq.~\ref{eq:gnn} can be further reformulated as:
\begin{equation}
\label{eq:mini-batch}
% \vspace{-1mm}
\begin{split}
    \mathbf{h}_{v}^{(\ell+1)} = \boldsymbol{f}_{\boldsymbol{\theta}}^{(\ell+1)} \Big( \mathbf{h}_{v}^{(\ell)}, 
    &\underbrace{\big\{\!\!\big\{\mathbf{h}^{(\ell)}_u \big\}\!\!\big\}_{u \in  \mathcal{N}(v)   \cap \mathcal{S}(v)}}_{\text{In-subgraph nodes}} \\ \cup &\underbrace{\big\{\!\!\big\{\mathbf{h}^{(\ell)}_u  \big\}\!\!\big\}_{u \in \mathcal{N}(v) \setminus \mathcal{S}(v)}}_{\text{Out-of-subgraph nodes}}  \Big),
\end{split}
% \vspace{-1mm}
\end{equation}
where $\mathcal{S}(v)$ denotes the subgraph that node $v$ belongs to. In this paper, we consider the distributed training of GNNs with multiple local machines and a global server. The original input graph $\mathcal{G}$ is first partitioned into $M$ subgraphs, where each $\mathcal{G}_{m}(\mathcal{V}_{m},\mathcal{E}_{m})$ represents the $m$-th subgraph. Our goal is to find the optimal parameter $\boldsymbol{\theta}$ in a distributed manner by minimizing the global loss:
\begin{equation}
    {\min}_{\boldsymbol{\theta}}\;\; \mathcal{L}_{\text{global}}\big(\boldsymbol{\theta}\big) = {\sum}_{m=1}^{M} \mathbf{w}_m\cdot\mathcal{L}_{\text{local}}^{(m)}\big(\boldsymbol{\theta}\big),
    \label{eq: global loss}
\end{equation}
where $\mathbf{w}_m$ denotes the averaging weights and for each local loss and the local losses are given by:
\begin{equation}
\begin{split}
    \label{eq: local loss}    \mathcal{L}^{(m)}_{\text{local}}\big(\boldsymbol{\theta}\big) = \frac{1}{\vert\mathcal{V}_m \vert}{\sum}_{v\in\mathcal{V}_m} Loss\big(\mathbf{h}_{v}^{(L)},\mathbf{y}_{v}\big), \;\forall m.
\end{split}
\end{equation}
% where $\{\boldsymbol{\theta}_{m}\}_{m=1}^M$ are parameters for each local GNN model and $\mathbf{h}_{v}^{(L)}$ follows Eq.~\ref{eq:mini-batch} recursively. At each round of communication, each local parameter is aggregated to update the global parameter, i.e., $\boldsymbol{\theta}=AGG(\boldsymbol{\theta}_{1},\boldsymbol{\theta}_{2},\cdots,\boldsymbol{\theta}_{M})$. The $AGG$ function can take the average sum, weighted average, etc.

Existing methods in distributed training for GNNs can be classified into two categories, namely "\emph{partition-based}" and "\emph{propagation-based}". The "Partition-based" method~\cite{angerd2020distributed,jia2020improving,ramezani2021learn} generalizes the existing data parallelism techniques of classical distributed training on \emph{i.i.d} data to graph data and enjoys minimal communication cost. However, the embeddings of neighbor nodes (``out-of-subgraph nodes'' in Eq.~\ref{eq:mini-batch}) are dropped and the connections between subgraphs are thus ignored, which results in severe information loss. Hence, another line of work, namely the "propagation-based" method~\cite{ma2019neugraph,zhu2019aligraph,zheng2020distdgl,tripathy2020reducing,wan2022pipegcn} considers using communication of neighbor nodes for each subgraph (``out-of-subgraph nodes'' in Eq.~\ref{eq:mini-batch}) to satisfy GNN's neighbor aggregation, which minimizes the information loss. However, due to the \emph{neighborhood explosion} problem, inevitable communication overhead is incurred and plagues the achievable training efficiency.

\section{Problem Formulation}
\label{sec: problem formulation}

We follow the partition-parallel distributed training of GNNs defined in Eq.~\ref{eq: global loss}. Given the $m$-th graph partition $\mathcal{G}_m$, we reformulate Eq.~\ref{eq:mini-batch} in the matrix form as:
\begin{equation}
    \mathbf{H}_{in}^{(\ell+1,m)} = \boldsymbol{f}_{\boldsymbol{\theta}_m}^{(\ell+1)}\Big(\mathbf{H}_{in}^{(\ell,m)},\mathbf{H}_{out}^{(\ell,m)}\Big),
\label{eq:matrix form forward}
\end{equation}
where $\mathbf{H}_{in}^{(\ell,m)}$ and $\mathbf{H}_{out}^{(\ell,m)}$ denotes the in- and out-of-subgraph node embeddings at $\ell$-th layer on partition $\mathcal{G}_m$, respectively. As mentioned earlier, directly swapping $\mathbf{H}_{out}^{(\ell,m)}$ between each subgraph will result in exponential communication costs and harm the concurrency of distributed training. Existing historical-value-based methods approximate the out-of-subgraph embeddings by historical embeddings $\mathbf{\Tilde{H}}_{out}^{(\ell,m)}$, which result in a staleness error, i.e.,
\begin{equation}
    \delta\mathbf{\Tilde{H}}^{(\ell,m)} \coloneqq \big\Vert\mathbf{H}_{out}^{(\ell,m)} - \mathbf{\Tilde{H}}_{out}^{(\ell,m)}  \big\Vert.
    \label{eq: staleness def}
\end{equation}
In this work, we consider predicting $\mathbf{H}_{out}^{(\ell,m)}$ in an efficient and data-driven manner such that the predicted embedding $\mathbf{\hat{H}}_{out}^{(\ell,m)}$ has a smaller staleness than $\mathbf{\Tilde{H}}_{out}^{(\ell,m)}$. 
Despite the necessity, how to handle the above problem is an open research area due to several existing challenges: 1). The underlying evolution of the true embedding $\mathbf{H}_{out}^{(\ell,m)}$ is unknown and complicated due to the nature of GNN's computation and distributed training;  2). How to design the algorithm to reduce the staleness without hurting the efficiency and scalability of distributed training is highly non-trivial. 3). How the added staleness alleviation strategies will impact the model training process, such as the convergence and stability, is a difficult yet important question to address.

\begin{figure*}[t!]
  \begin{center}
    \includegraphics[width=0.98\textwidth]{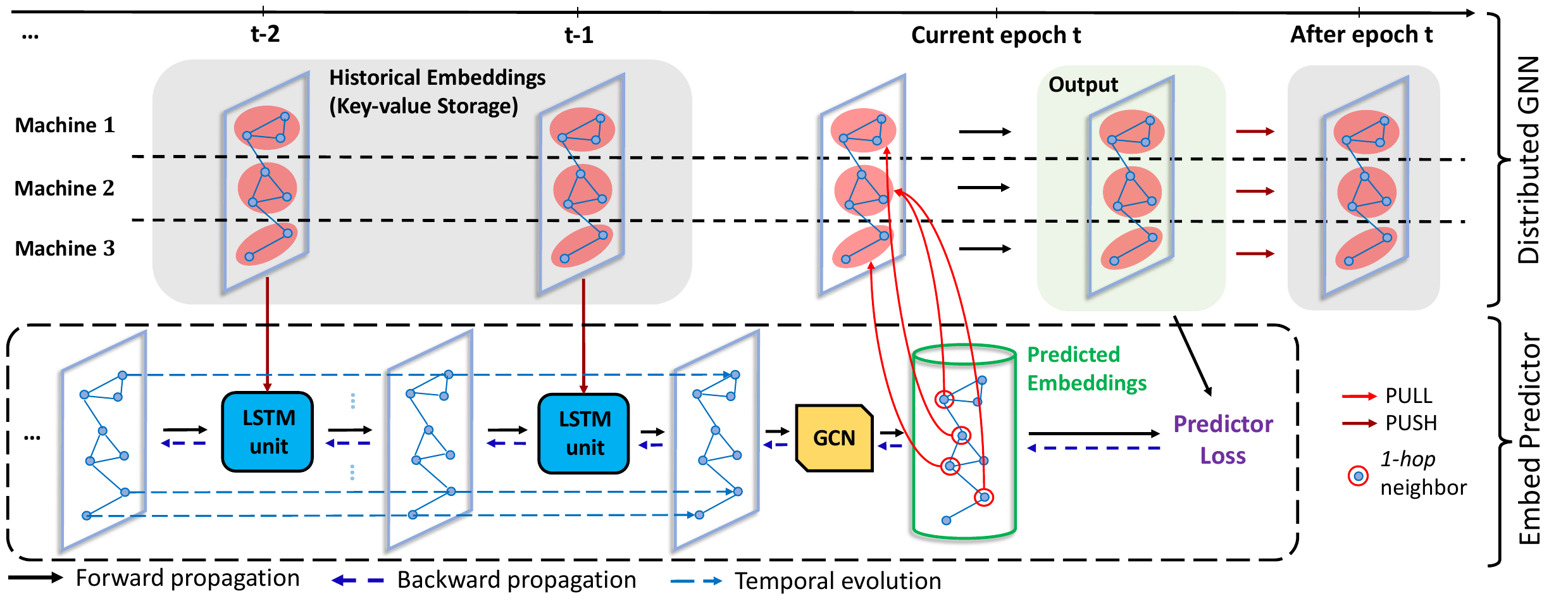}
  \end{center}
  \vspace{-3mm}
  \caption{\textbf{Overview of our proposed SAT framework} (per-GNN-layer view). The upper body depicts the parallel distributed training of the GNN, and the lower body depicts that our embedding predictor reduces the embedding staleness by modeling how embeddings evolve temporally and spatially. The predicted embeddings are pulled to each machine for highly concurrent training of the GNN, and its output during the forward pass serves as weak supervision to train the embedding predictor. Such a design decouples the computation of these two components and allows us to train them in an efficient online manner. Our proposed embedding predictor is general and method-agnostic, i.e., holding the potential for enhancing the quality and relevance of historical embeddings, which could lead to improvements in various applications and domains.}
  \label{fig: framework figure}
  \vspace{-3mm}
\end{figure*}

\section{Proposed Method}

This section introduces our framework Staleness-Alleviated Training (SAT) that jointly addresses the challenges above. We first present an overview of SAT, followed by introducing our proposed embedding predictor, a novel architecture that adaptively captures the evolution of node embeddings with specially designed input data and training objectives. Finally, we demonstrate the proposed online algorithm for optimizing the framework with a theoretical convergence guarantee.

\subsection{Overview of SAT}

Figure~\ref{fig: framework figure} provides the detailed end-to-end flow of SAT, where the embedding predictor reduces the staleness of historical embeddings by leveraging their evolution pattern over past epochs. This is achieved by modeling distributed GNN's embeddings as \emph{temporal graphs} such that the embedding predictor jointly captures the spatial and temporal evolution. The predicted embeddings serve as a better approximation of ground truth and are pulled by each machine as additional inputs in their parallel forward propagation, which can be expressed as:
\begin{equation}
\begin{split}
    \mathbf{H}_{in}^{(\ell+1,m)} \approx \boldsymbol{f}_{\boldsymbol{\theta}_{m}}^{(\ell+1)}\Big(\mathbf{H}_{in}^{(\ell,m)}, \underbrace{\boldsymbol{g}_{\boldsymbol{\omega}} \big(\mathcal{G}_{tmp}^{(m)}\big)}_{\text{Predicted embeddings}}\Big),
\end{split}
\label{eq: sat forward}
\end{equation}
where $\boldsymbol{g}_{\boldsymbol{\omega}}(\cdot)$ denotes the embedding predictor, and $\mathcal{G}_{tmp}^{(m)}$ represents the temporal graph for the $m$-th partition. Note that the pulled predicted embeddings are only for out-of-subgraph nodes hence the communication cost is kept low. The entire framework can be trained via an online algorithm.

\subsection{Embedding Predictor} 

Here we introduce our proposed embedding predictor $\boldsymbol{g}_{\boldsymbol{\omega}}(\cdot)$ which aims to alleviate the staleness error defined in Eq.~\ref{eq: staleness def}. Our contribution is \emph{2-fold}: First, we innovatively formulate the embedding prediction task as modeling the temporal graphs induced by the distributed GNN over different epochs. The temporal graphs fully characterize the underlying evolution of node embeddings, thus enabling our predictor to predict the target embeddings with sufficient information. Second, to jointly optimize the embedding predictor and the induced temporal graphs, we propose a new multi-task learning loss where each task corresponds to the embedding prediction on each graph partition.

\noindent\textbf{Temporal Graphs Induced by Distributed GNNs.}
Recall the forward propagation of distributed GNN defined in Eq.~\ref{eq:matrix form forward}, our goal is to predict the out-of-subgraph embeddings for the current training epoch such that the predicted embeddings offer a better approximation than historical embeddings. Intuitively, every local GNN recursively computes the node embeddings by message passing across multiple layers, while the usage of historical embeddings further introduces the previous epoch's information. Hence, our goal requires the proposed embedding predictor to jointly capture the spatial (\emph{across-layer}) and temporal (\emph{across-epoch}) evolution of the node embeddings.

To this end, we need to define the input data for the embedding predictor such that the data sufficiently characterizes the underlying evolution of embeddings. To see this, consider the subgraph $\mathcal{\Bar{G}}_{m}= (\mathcal{\Bar{V}}_{m},\mathcal{\Bar{E}}_{m})$ induced by the $m$-th graph partition $\mathcal{G}_{m}$ and its 1-\emph{hop} neighborhoods, where the edges between them are preserved as in the original raw graph.
Denote the in- and out-of-subgraph embeddings computed by the local GNN model at epoch $t$ as $\mathbf{H}_{in}^{(t,\ell,m)}$ and $\mathbf{H}_{out}^{(t,\ell,m)}$, where $\ell$ goes from $1$ to $L$. 
The key observation here is that if we plug these node embeddings into $\mathcal{\Bar{G}}_{m}$ as node weights, we can define a \emph{temporal graph} $\{\mathcal{\Bar{G}}^{(t)}_{m}\}_{t\in\mathcal{T}}$ as:
\begin{equation}
\begin{split}
    \big\{\mathcal{\Bar{G}}^{(t)}_{m}\big\}_{t\in\mathcal{T}} \coloneqq \Big\{ & \Big(\mathcal{\Bar{V}}_{m},\mathcal{\Bar{E}}_{m},\{\mathbf{H}_{in}^{(t,\ell,m)}\}_{\ell=1}^L,  \\
    &\{\mathbf{H}_{out}^{(t,\ell,m)}\}_{\ell=1}^L\Big): t=1,2,\cdots,T\Big\},
\end{split}
\label{eq: temporal graph def}
\end{equation}
where $T$ denotes the total number of epochs. In temporal graph $\{\mathcal{\Bar{G}}^{(t)}_{m}\}_{t\in\mathcal{T}}$, each node and edge is the same as in the original full graph, and the timestamp $t$ is defined as each training epoch.
Based on Eq.~\ref{eq: temporal graph def}, the entire set of temporal graphs induced by the distributed GNN can be defined as $\{\mathcal{\Bar{G}}^{(t)}\}_{t\in\mathcal{T}} \coloneqq \big\{\{\mathcal{\Bar{G}}^{(t)}_{m}\}_{t\in\mathcal{T}}: m=1,2,\cdots, M \big\}$,
% \begin{equation}
%     \big\{\mathcal{\Bar{G}}^{(t)}\big\}_{t\in\mathcal{T}} \coloneqq \Big\{\big\{\mathcal{\Bar{G}}^{(t)}_{m}\big\}_{t\in\mathcal{T}}: m=1,2,\cdots, M \Big\},
% \label{eq: global temporal graph}
% \end{equation}
where $M$ denotes the number of subgraphs.

\noindent\textbf{Multi-task Learning of Embedding Predictor.}
Given the temporal graph defined above, our \emph{goal} is to build an embedding predictor that proactively captures the evolving pattern and predicts embeddings for the current epoch.
The key observation here is that each subgraph induces a temporal graph and we want to jointly train an embedding predictor that is able to make predictions for any subgraphs. Multi-task Learning (MTL~\cite{zhang2021survey}), which jointly trains a model on multiple different tasks to improve the generalization ability, provides a suitable option for our problem. 
Formally, to learn the embedding predictor $\boldsymbol{g}_{\boldsymbol{\omega}}$ parameterized by $\boldsymbol{\omega}$ we optimize the following objective at epoch $t$:
\begin{equation}
\begin{split}
    {\min}_{\boldsymbol{\omega}_t}\;\; \frac{1}{M}  & \sum_{m=1}^{M} \Big\Vert \boldsymbol{g}_{\boldsymbol{\omega}_t} \big(\{\mathcal{\Bar{G}}^{(s)}_{m}\}_{t-\tau\leq s\leq t-1}\big) - \mathbf{H}^{(t,m)} \Big\Vert, \\
    &\text{s.t.,}\;\; \delta\mathbf{\Hat{H}}^{(t,m)} < \delta\mathbf{H}^{(t-1,m)},\;\forall m,
\end{split}
\label{eq: problem formulation}
\end{equation}
where $\mathbf{H}^{(t,m)}$  denotes the concatenation of in- and out-of-subgraph embeddings without any staleness, $\mathbf{\Hat{H}}^{(t,m)}$ is a compact notation of the output by $\boldsymbol{g}_{\boldsymbol{\omega}_t}$, and $\tau$ denotes the length of the \emph{sliding window} where we restrain the mapping function can only access up to $\tau$ steps of historical information.

In this work, we consider combining GNNs with \emph{recurrent structures} as embedding predictors. The GNN captures the information within the node dependencies while the recurrent structure captures the information within their temporal evolution.
Specifically, we consider implementing our embedding predictor $\boldsymbol{g}_{\boldsymbol{\omega}}$ as an RNN-GNN~\cite{wu2022graph}. The combination of LSTM and GCN is what we found empirically the optimal trade-off between efficiency and capacity in most cases, while the RNN-GNN is generic for many different variants (LSTM-GAT, GRU-GCN, etc.)
% which is defined as follows:
% \begin{equation}
%         \mathbf{S}^{(t,m)}, \mathbf{C}^{(t,m)} = LSTM\big(\mathbf{H}^{(t-1,m)}, \mathbf{S}^{(t-1,m)}, \mathbf{C}^{(t-1,m)}  \big), \quad
%         \Delta\mathbf{\hat{H}}^{(t,m)} = GCN\big(\mathbf{S}^{(t,m)}, \mathbf{A}\big),
% \label{eq:aux rnn-gnn}
% \end{equation}
% where $\mathbf{S}$ and $\mathbf{C}$ denote the hidden state and memory cell for the LSTM unit, respectively. 
% The computation of LSTM in the equation above iterates from $t-\tau$ to $t-1$ and we only show the last step for simplicity.

\subsection{Optimization Algorithm}

Here we demonstrate the general loss function of SAT and the proposed optimization algorithm. The detailed algorithm is shown in Algorithm~\ref{alg: sat}.
With the additional embedding predictor, the training procedure for SAT is more complicated than standard distributed GNN training. As a result, jointly optimizing the GNN and embedding predictor is tricky because the forward passes of the two models are \emph{coupled}. To see this, the overall loss of SAT can be expressed as a \emph{nested} optimization as follows:
\begin{equation}
\begin{split}
    \forall m, \;\; &{\boldsymbol{\theta}_{m}^{*}} = \arg{\min}_{{\boldsymbol{\theta}_m}} \;\mathcal{L}^{(m)}_{\text{local}}\big(\boldsymbol{\theta}_{m}, \boldsymbol{\omega}_{t}^{*} \big), \\
    &\text{s.t.,}\;\; \boldsymbol{\omega}_{t}^{*} = \arg{\min}_{\boldsymbol{\omega}_{t}} \;\mathcal{L}_{\text{Pred}}\big(\boldsymbol{\omega}_{t}, \{\boldsymbol{\theta}_{m}^{*}\}_{m=1}^M \big).  
\end{split}
\label{eq: overall loss}
\end{equation}
By following Eq.~\ref{eq: global loss} and plugging our embedding predictor into Eq.~\ref{eq:mini-batch}, the distributed GNN's loss is:
\begin{equation}
\begin{split}
\mathcal{L}^{(m)}_{\text{local}}\big(\boldsymbol{\theta}_{m}, \boldsymbol{\omega}_{t}^{*}\big) = \frac{1}{\vert\mathcal{V}_m \vert}{\sum}_{v\in\mathcal{V}_m} Loss\big(\mathbf{h}_{v}^{(L)},\mathbf{y}_{v}\big),        
\end{split}
\label{eq: distributed gnn loss}
\end{equation}
where it recursively satisfies:
\begin{equation}
\begin{split}
    \mathbf{h}_{v}^{(L)} =& \boldsymbol{f}_{\boldsymbol{\theta}_{m}}^{(L)} \Big(\big\{\!\!\big\{ \mathbf{h}^{(L-1)}_u  \big\}\!\!\big\}_{u \in \mathcal{N}(v)  \cap \mathcal{S}(v)}, \\
    &\big\{\!\!\big\{\boldsymbol{g}_{\boldsymbol{\omega}_{t}^{*}} \big(\{\mathcal{\Bar{G}}^{(s)}_{m}\}_{t-\tau\leq s\leq t-1}\big)\big\}\!\!\big\}_{u \in \mathcal{N}(v)  \setminus \mathcal{S}(v)}^{L-1} \Big).
\end{split}
\end{equation}
By plugging the distributed GNN's forward pass into Eq.~\ref{eq: problem formulation}, the embedding predictor's loss is:
\begin{equation}
\begin{split}
    \mathcal{L}_{\text{Pred}}\big(\boldsymbol{\omega}_{t}, \{\boldsymbol{\theta}_{m}^{*}\} \big)  &= \frac{1}{M}{\sum}_{m} \Big\Vert \boldsymbol{g}_{\boldsymbol{\omega}_t} \Big(\{\mathcal{\Bar{G}}^{(s)}_{m}\}_{t-\tau}^{t-1}\Big) - \\
     \Big\{&\boldsymbol{f}_{\boldsymbol{\theta}_{m}^{*}}^{(\ell+1)} \big(\mathbf{H}_{in}^{(t,\ell,m)},\mathbf{\Hat{H}}_{out}^{(t.\ell,m)}\big) \Big\}_{\ell=0}^{L-1}\Big\Vert.
\end{split}
\label{eq: predictor loss}
\end{equation}
The key observation here is that the predicted embeddings $\boldsymbol{g}_{\boldsymbol{\omega}_t} (\{\mathcal{\Bar{G}}^{(s)}_{m}\}_{t-\tau\leq s\leq t-1})$ have a \emph{2-fold} functionality: 1) it serves as an additional input in the forward propagation of the distributed GNN, and 2) it serves as the prediction to calculate the loss function of the embedding predictor. Noticing this, we propose an online algorithm to train the GNN model and the embedding predictor \emph{alternatively}, which decouples the 2-fold functionality and allows easier optimization. The pseudo-code for SAT's training procedure is provided in Algorithm~\ref{alg: sat}, and a more detailed optimization process of our online algorithm can be found in the appendix.

% As shown in Eq.~\ref{eq:aux rnn-gnn} and Eq.~\ref{eq: sat forward}, our design of the embedding predictor borrows the idea of skip-connection~\cite{he2016identity}, and the prediction of the embedding predictor can be interpreted as an \emph{error compensation} term which corrects the staleness of previous epoch embedding. [\textcolor{blue}{add discussion.}]

\begin{algorithm}[t!]
\caption{Staleness-alleviated Distributed GNN Training}\label{alg: sat}
\begin{algorithmic}[1]
\Require GNN learning rate $\eta_1$, embedding predictor learning rate $\eta_2$, update frequency $\Delta T$.
% \Ensure $y = x^n$
% \State Initialize global GCN parameters $\boldsymbol{\theta}^{(1)}$
\State /*\; \emph{Partitioning the raw graph} \;*/
\State $\{\mathcal{G}_{m}(\mathcal{V}_{m},\mathcal{E}_{m}), m = 1,..,M\}$ $\leftarrow$ METIS($\mathcal{G}$) 
\For{$t = 1...T$}
    \For{$m=1...M$ in parallel}   
        % \State Get local GCN parameters $\boldsymbol{\theta}^{(t)}_m = \boldsymbol{\theta}^{(t)}$
        \For{$\ell = 1...L$}
            % \If{$t\;\%\; N == 0$ and $\ell \neq L$}
            \State Pull $\mathbf{\Hat{H}}_{out}^{(\ell,m)}$ to local machines    

            % \For{$v \in \mathcal{V}_m$}
            % % // forward propagation \\
            % \State $\mathbf{h}_{out}^{(\ell)} = \{\mathbf{h}^{(\ell)}_u :  u \in \mathcal{N}(v) \setminus \mathcal{V}_m\}$
            % \State $\mathbf{h}_{in}^{(\ell)} = \{\mathbf{h}^{(\ell)}_u :  u \in \mathcal{N}(v) \cap \mathcal{V}_m\}$
            % \State $\mathbf{h}^{(\ell+1)}_v = \sigma \Big(\boldsymbol{\theta}^{(t, \ell)}_m\cdot$ CONCAT$\big(\mathbf{h}^{(\ell)}_v, \mathbf{h}_{in}^{(\ell)},  \mathbf{h}_{out}^{(\ell)}\big)\Big)$
            \State Forward prop for local GNNs as Eq.~\ref{eq: sat forward}   
        % PUSH($h^l_v$) \algorithmiccomment{push stale representations} \\             
        % \EndFor
        % \If{$(t - 1)\;\%\;N == 0$ and $\ell \neq L$}
             \State Push computed embeddings to the server    
             
             % $\mathbf{H}^{(\ell,m)}_{in} \rightarrow \mathbf{\tilde{H}}^{(\ell,m)}_{in} $                % \algorithmiccomment{PUSH}
                    
            % \EndIf
            % \State Normalize embeddings $\mathbf{h}^{(\ell)}_v \leftarrow \mathbf{h}^{(\ell)}_v/\|\mathbf{h}^{(\ell)}_v\|_2$, $\forall\;v \in \mathcal{V}_m$ \;

    \EndFor
         \State /*\; \emph{Local GNN update} \;*/
         \State Compute local loss as Eq.~\ref{eq: distributed gnn loss} and gradients $\bigtriangledown \boldsymbol{\theta}^{(t)}_m$
         \State $\boldsymbol{\theta}^{(t+1)}_m = \boldsymbol{\theta}^{(t)}_m - \eta_1 \cdot \bigtriangledown \boldsymbol{\theta}^{(t)}_m$
    \EndFor
        \State /*\; \emph{Global GNN update} \;*/
        \State $\boldsymbol{\theta}^{(t+1)} \leftarrow$ \textbf{AGG}($\boldsymbol{\theta}^{(t + 1)}_1...\boldsymbol{\theta}^{(t + 1)}_M$) 
        \State /*\; \emph{Embedding predictor update} \;*/
        \If{$t \;\%\; \Delta T ==  0$}  
        \State Compute embedding predictor loss as Eq.~\ref{eq: predictor loss}  

        \State $\boldsymbol{\omega}_{t+1}= \boldsymbol{\omega}_{t} - \eta_2 \cdot \bigtriangledown \mathcal{L}_{Pred}(\boldsymbol{\omega}_t)$
        % \State Compute error compensation $\Delta\mathbf{\hat{H}}_{out}^{(\ell,m)}$ following Eq.~\ref{eq:aux rnn-gnn}
        \State Update predicted embeddings by $g_{\boldsymbol{\omega}_{t+1}}$
        % \State $\mu \leftarrow t$
        \EndIf
\EndFor
\end{algorithmic}
\end{algorithm}

%%%%%%%%%%%%%%%%%%%%%%%%%
%%%%%%%%%%%%%%%%%%%%%%%%%

\noindent\textbf{Optimize distributed GNN $\{\boldsymbol{\theta}_{m}\}$.}
Given embedding predictor parameters $\boldsymbol{\omega}_{t}^{*}$, the predicted out-of-subgraph embeddings (last term in the second equation of Eq.~\ref{eq: distributed gnn loss} becomes \emph{constant}, since the temporal graphs contain historical information which can be regarded as additional input in the current epoch. Hence, the optimization of Eq.~\ref{eq: distributed gnn loss} with respect to $\{\boldsymbol{\theta}_{m}\}$ can be directly solved by gradient descent algorithms. As we define our temporal graphs for each graph partition, Eq.~\ref{eq: distributed gnn loss} can still be solved in parallel between $m$. Meanwhile, the backpropagation of Eq.~\ref{eq: distributed gnn loss} actually involves ALL neighbors' information (though predicted instead of ground truth) since the gradient also depends on the predicted out-of-subgraph embeddings. In other words, both in- and out-of-subgraph node dependencies are considered during the backpropagation.

\noindent\textbf{Optimize embedding predictor $\boldsymbol{\omega}_{t}$.}
Similarly, given $\boldsymbol{\theta}_{m}^*$, the optimization of $\boldsymbol{\omega}_{t}$ can be done by gradient descent on Eq.~\ref{eq: predictor loss}. Here we introduce two techniques to improve the training efficiency of the embedding predictor. First, it has been shown that \emph{early stopping} in SGD can be regarded as implicit regularization~\cite{rajeswaran2019meta}, and we adopt this technique over $\boldsymbol{\omega}_{t}$ to help reduce the computational cost to train the predictor. 
To further enhance the efficiency and flexibility of the training, we extend the loss to \emph{stochastic} version by sampling mini-batches as:
\begin{equation}
    {\min}_{\boldsymbol{\omega}_t}  \frac{1}{M}\sum_{m} \frac{1}{\vert \mathcal{B}_m \vert} \sum_{u\in\mathcal{B}_m} \mathcal{L}_{\text{Pred}}^{(m)}\big(\boldsymbol{\omega}_{t}, \boldsymbol{\theta}_{m}^{*} \big).
    \label{eq: mini-batch training}
\end{equation}

\subsection{Data Compression for Historical Embedding Storage}

In SAT, our embedding predictor aims to reduce the staleness of historical embeddings based on the induced temporal graph, by optimizing the following loss function:
\begin{equation}
\begin{split}
    {\min}_{\boldsymbol{\omega}_t}\;\; \frac{1}{M}\sum_{m=1}^{M} & \Big\Vert \boldsymbol{g}_{\boldsymbol{\omega}_t} \big(\{\mathcal{\Bar{G}}^{(s)}_{m}\}_{t-\tau\leq s\leq t-1}\big) - \mathbf{H}^{(t,m)} \Big\Vert,\\
    &\text{s.t.,}\;\; \delta\mathbf{\Hat{H}}^{(t,m)} < \delta\mathbf{H}^{(t-1,m)}.
\end{split}
% \label{eq: problem formulation}
\end{equation}
Although we introduce the hyperparameter $\tau$ as the length of the \emph{sliding window} to obtain a constant memory complexity concerning $t$, the memory cost of the temporal graph can still be prohibitive when the original graph is huge. To this end, we borrow the idea of data compression to reduce the cost of storage for historical embeddings.

In this work, we design a simple yet effective compression scheme based on the \emph{Encoded Polyline Algorithm}~\footnote{\url{https://developers.google.com/maps/documentation/utilities/polylinealgorithm}} (or polyline encoding.) Polyline encoding is a lossy compression algorithm that converts a rounded binary value into a series of character codes for ASCII characters using the base64 encoding scheme. It can be configured to maintain a configurable precision by rounding the value to a specified decimal place. The model could achieve the largest communication savings and minor performance loss with the appropriate precision. We empirically found that, by choosing proper parameters for the compression algorithm, we can achieve up to \textbf{3.5$\times$} of compression ratio for the historical embeddings. Hence, although historical embeddings of more than 1 epoch are stored, the overall memory cost of SAT is still comparable to the baselines (please refer to Table 2 of our main text.)

As a byproduct, the compression algorithm can also help reduce the communication cost of SAT. To see this, the pull and push operations as shown in our Algorithm 1 (main text) can both enjoy a reduced communication cost if we compress the node embeddings by using the polyline encoding before triggering the pull/push operations. After the communication, we can decompress the embeddings and follow the procedures as defined in SAT.

\subsection{Theoretical Analyses}

\noindent\textbf{Convergence Analyses.}
Here we analyze the convergence of SAT. We show that the global GNN model can converge under the distributed training setting with the embedding predictor. Due to limited space, more details such as the proof can be found in Section~\ref{sec: theory proof}.

\begin{theorem}
Consider a GCN with $L$ layers that are $L_f$-Lipschitz smooth. $\forall\;\epsilon > 0$, $\exists$ constant $E>0\;$ such that, we can choose a learning rate $\eta=\sqrt{M\epsilon}/{E}$ and number of training iterations $T=(\mathcal{L}(\boldsymbol{\theta}^{(1)})-\mathcal{L}(\boldsymbol{\theta}^{*}))E M^{-\frac{1}{2}}\epsilon^{-\frac{3}{2}}$, $s.t.$,
% $T^{-1}\sum\nolimits_{t=1}^T \Vert \nabla \mathcal{L}(\boldsymbol{\theta}^{(t)}) \Vert^{2} \leq \mathcal{O}(T^{-\frac{2}{3}}M^{-\frac{1}{3}})$,
\begin{equation}
% \vspace{-2mm}
    \frac{1}{T}\sum_{t=1}^T \left\Vert \nabla \mathcal{L}(\boldsymbol{\theta}^{(t)}) \right\Vert^{2}_{2} \leq \mathcal{O}\left(\frac{1}{T^{\frac{2}{3}}M^{\frac{1}{3}}}\right),
\end{equation}
where $\boldsymbol{\theta}^{*}$ denotes the optimal parameter for the global GCN.
\label{thm: sync convergence}
% \vspace{-3mm}
\end{theorem}

\noindent\textbf{Communication Cost Analyses.} Here we analyze the communication cost of SAT as depicted in Algorithm~\ref{alg: sat}. 
Without loss of generality, consider a GNN with $L$ layers and a fixed width $d$. 
SAT's communication cost \emph{per round} can be expressed as:
\begin{equation}
    \mathcal{O}\big(M L d^2 + \sum\nolimits_{m=1}^{M}\left\vert \cup_{v\in\mathcal{V}_{m}}\mathcal{N}(v)\setminus \mathcal{V}_m \right\vert L d + N L d\big),
    \label{eq: communication complexity}
\end{equation}
where the first term denotes the cost for pull/push of GNN parameters or gradients, the second term and the third term represent the cost for pull and push of embeddings, where $N$ is the raw graph size. 
To reduce the accumulated communication and computation cost, we consider decreasing the frequency of finetuning the embedding predictor by a factor every $\Delta T$ epoch (Line 17 in Algorithm~\ref{alg: sat}). Such a design for periodic updates of the embedding predictor can greatly reduce the computational and communication overhead by a factor of $\Delta T$.

\begin{table}[t!]
\scriptsize
\centering
\caption{Summary of dataset statistics.}
\label{tab: dataset}
% \resizebox{\textwidth}{!}{%
\begin{tabular}{llllll}
\hline
Dataset      & \#~Nodes   & \#~Edges     & \#~Features & \#~Classes \\ \hline
% Cora         & 2,708      & 10,556        & 1,433      & 7            \\
% Citeseer     & 3,327	     & 9,104	   & 3,703	    & 6            \\
% Pubmed       & 19,717	 & 88,648	   & 500	    & 3            \\
Flickr       & 89,250    & 899,756     & 500        & 7            \\
Reddit       & 232,965   & 23,213,838  & 602        & 41           \\
OGB-Arixv    & 169,343   & 2,315,598   & 128        & 40           \\
OGB-Products & 2,449,029 & 123,718,280 & 100        & 47           \\ 
OGB-Papers100M & 111,059,956 & 1,615,685,872 & 128        & 172           \\ 
\hline
\end{tabular}%
\end{table}

\begin{table*}[t!]
\small
\centering
\caption{Performance comparison (test F1 score) of distributed GNN frameworks. The mean and standard deviation are calculated based on multiple runs with varying seeds. The subscript under each dataset's name denotes the number of partitions we used. The best and second best methods are marked with \textbf{bold} and \underline{underline}, respectively. SAT consistently boosts performance by alleviating the staleness of historical embeddings. Also, our SAT favors the bigger number of partitions.}
\begin{tabular}{llccccccc}
\toprule
\textbf{Backbone} & \textbf{Method} & \textbf{Flickr$_{(4)}$} & \textbf{Reddit$_{(4)}$} & \textbf{ogb-arxiv$_{(4)}$} & \textbf{ogb-products$_{(8)}$} & \textbf{Flickr$_{(8)}$} & \textbf{Reddit$_{(8)}$} & \textbf{ogb-arxiv$_{(8)}$} \\
\midrule 
\multirow{10}{*}{GCN} 
& LLCG    & 50.73\scriptsize$\pm$0.15  &  62.09\scriptsize$\pm$0.41  &  69.80\scriptsize$\pm$0.21   &  76.87\scriptsize$\pm$0.32   & 50.53\scriptsize$\pm$0.20 &  61.80\scriptsize$\pm$0.38 &  69.62\scriptsize$\pm$0.24 \\
& DistDGL   & 50.90\scriptsize$\pm$0.13  &  87.02\scriptsize$\pm$0.23 & 69.90\scriptsize$\pm$0.17 &  77.52\scriptsize$\pm$0.28   & 50.70\scriptsize$\pm$0.18 &  86.82\scriptsize$\pm$0.28 &  69.70\scriptsize$\pm$0.19  \\
& BNS-GCN   & 51.76\scriptsize$\pm$0.21  &  93.76\scriptsize$\pm$0.43 & 55.49\scriptsize$\pm$0.17 &  78.47\scriptsize$\pm$0.31   & 51.56\scriptsize$\pm$0.19 &  96.56\scriptsize$\pm$0.39 & 55.39\scriptsize$\pm$0.14  \\
& SANCUS   & 52.17\scriptsize$\pm$0.31  &  93.81\scriptsize$\pm$0.36 & 70.05\scriptsize$\pm$0.25 &  78.95\scriptsize$\pm$0.36   & 52.07\scriptsize$\pm$0.29 &  93.61\scriptsize$\pm$0.33 & 69.85\scriptsize$\pm$0.28 \\
& DistGNN & 51.89\scriptsize$\pm$0.21  &  94.23\scriptsize$\pm$0.43  &  71.54\scriptsize$\pm$0.23   &  77.36\scriptsize$\pm$0.42    & 53.38\scriptsize$\pm$0.17 &  95.03\scriptsize$\pm$0.39 &  71.90\scriptsize$\pm$0.19  \\
& \; \emph{w/ SAT (Ours)}    & 52.73\scriptsize$\pm$0.19 &  95.21\scriptsize$\pm$0.44 &  \underline{72.13\scriptsize$\pm$0.37}  & 78.97\scriptsize$\pm$0.33  & \textbf{53.81\scriptsize$\pm$0.24} &  95.21\scriptsize$\pm$0.41 &  \underline{72.03\scriptsize$\pm$0.3}  \\
& DIGEST    & 52.92\scriptsize$\pm$0.32 &  94.55\scriptsize$\pm$0.37   &  71.90\scriptsize$\pm$0.16   & 78.56\scriptsize$\pm$0.29        & 53.00\scriptsize$\pm$0.28 &  94.35\scriptsize$\pm$0.34 &  71.80\scriptsize$\pm$0.23  \\
&  \; \emph{w/ SAT (Ours)}   &  \underline{53.45\scriptsize$\pm$0.27}  &  95.25\scriptsize$\pm$0.35   &   \textbf{72.39\scriptsize$\pm$0.21}   &   \underline{79.43\scriptsize$\pm$0.30}     &     52.36\scriptsize$\pm$0.18  &  94.49\scriptsize$\pm$0.26    &    \textbf{72.32\scriptsize$\pm$0.21}  \\
& PipeGCN    &  52.19\scriptsize$\pm$0.26 & \underline{96.20\scriptsize$\pm$0.38}  &   53.42\scriptsize$\pm$0.19 & 79.23\scriptsize$\pm$0.28      & 52.09\scriptsize$\pm$0.24 &  \underline{96.77\scriptsize$\pm$0.36} &  53.32\scriptsize$\pm$0.17  \\
&  \; \emph{w/ SAT (Ours)}    & 53.39\scriptsize$\pm$0.19   &    \textbf{97.02\scriptsize$\pm$0.31}      &    55.50\scriptsize$\pm$0.24  &          \textbf{80.21\scriptsize$\pm$0.34}   & 53.29\scriptsize$\pm$0.17 &  \textbf{96.92\scriptsize$\pm$0.34} &  53.45\scriptsize$\pm$0.15    \\
\midrule
\multirow{7}{*}{GAT} 
& LLCG   &  48.69\scriptsize$\pm$0.32     & 91.10\scriptsize$\pm$0.17     &  68.84\scriptsize$\pm$0.22    &  76.87\scriptsize$\pm$0.32   &  47.95\scriptsize$\pm$0.38     & 91.03\scriptsize$\pm$0.18     &  68.75\scriptsize$\pm$0.29 \\
& DistDGL   &  51.50\scriptsize$\pm$0.27    & 92.88\scriptsize$\pm$0.15  &  70.44\scriptsize$\pm$0.20 &  77.72\scriptsize$\pm$0.32   &  51.39\scriptsize$\pm$0.31    & 92.80\scriptsize$\pm$0.18  &  70.32\scriptsize$\pm$0.24  \\
& SANCUS   & 52.09\scriptsize$\pm$0.29    & 93.46\scriptsize$\pm$0.21  &  66.64\scriptsize$\pm$0.37 &  78.75\scriptsize$\pm$0.33   & 51.96\scriptsize$\pm$0.27    & 93.41\scriptsize$\pm$0.23  &  66.59\scriptsize$\pm$0.33 \\
& DistGNN &   52.38\scriptsize$\pm$0.25   & 94.29\scriptsize$\pm$0.33  &  68.08\scriptsize$\pm$0.38  &  78.26\scriptsize$\pm$0.26   &   52.31\scriptsize$\pm$0.29   & 94.20\scriptsize$\pm$0.21  &  67.93\scriptsize$\pm$0.31  \\
& \; \emph{w/ SAT (Ours)}    & 53.11\scriptsize$\pm$0.22 &  94.75\scriptsize$\pm$0.30 &  70.03\scriptsize$\pm$0.32 & 79.13\scriptsize$\pm$0.29  &     53.10\scriptsize$\pm$0.21     & 94.69\scriptsize$\pm$0.18  &  68.34\scriptsize$\pm$0.30  \\
& DIGEST    & 53.25\scriptsize$\pm$0.35 &  94.39\scriptsize$\pm$0.18   &  71.70\scriptsize$\pm$0.18   & 77.06\scriptsize$\pm$0.39         & 53.09\scriptsize$\pm$0.15    & 94.15\scriptsize$\pm$0.21 &  69.14\scriptsize$\pm$0.17  \\
&  \; \emph{w/ SAT (Ours)}   &  \textbf{53.65\scriptsize$\pm$0.32}  &  95.11\scriptsize$\pm$0.35   &   71.89\scriptsize$\pm$0.25   &   78.57\scriptsize$\pm$0.30    & \underline{53.57\scriptsize$\pm$0.23} &  95.02\scriptsize$\pm$0.27 &  71.77\scriptsize$\pm$0.21  \\
\bottomrule
\end{tabular}%
\label{tab: distributed training results table}
%\end{adjustwidth}
\end{table*}

\section{Experiment}
\label{sec:experiment}

In this section, we evaluate our proposed framework SAT with various experiments. For all distributed experiments, we simulate a distributed training environment using an EC2 {\texttt{g4dn.metal}} virtual machine (VM) instance on AWS, which has $8$ NVIDIA T4 GPUs, $96$ vCPUs, and $384$~GB main memory. We implemented the shared-memory KVS using the Plasma~\footnote{\url{https://arrow.apache.org/docs/python/plasma.html}} for embedding storage and retrieval. Our source code for the preprint version can be found at~\url{https://anonymous.4open.science/r/SAT-CA0C}.~\footnote{Due to the preprint status of our work, we are currently releasing only a partial version of our paper's resource code to prevent potential misuse. The full version will be made available upon publication. Thank you for your understanding.}

\subsection{Experiment Setting}

\noindent\textbf{Datasets.}
We evaluate SAT and other methods on 5 widely used large-scale node classification graph benchmarks:
\begin{itemize}[noitemsep,leftmargin=*]
    \item OGB-Arxiv~\cite{hu2020open} is a graph dataset that represents the citation network between arXiv papers. Within OGB-Arxiv, each node corresponds to an individual arXiv paper, while every edge indicates a citation relationship between the papers. The task is to predict the areas of papers.
    \item Flickr~\cite{zeng2019graphsaint} is used for categorizing types of images. In this dataset, each node is an individual image, and whenever two images share common properties, such as location, an edge is established between them. 
    \item Reddit~\cite{zeng2019graphsaint} is a graph dataset from Reddit posts. The node label is the community that a post belongs to. Two posts are connected if the same user comments on both posts.
    \item OGB-Products~\cite{hu2020open} representing an Amazon product co-purchasing network. Nodes represent products while edges between two products indicate that the products are purchased together. The task is to predict the categories of products. 
    \item OGB-Papers100m~\cite{hu2020open} is a citation graph, and it is used for predicting the subject areas of papers.
\end{itemize}
The detailed information such as statistics of these datasets is summarized in Table~\ref{tab: dataset}.

\begin{figure*}[t!]
\centering
  \begin{subfigure}{0.24\textwidth}
    \includegraphics[width=1.6in]{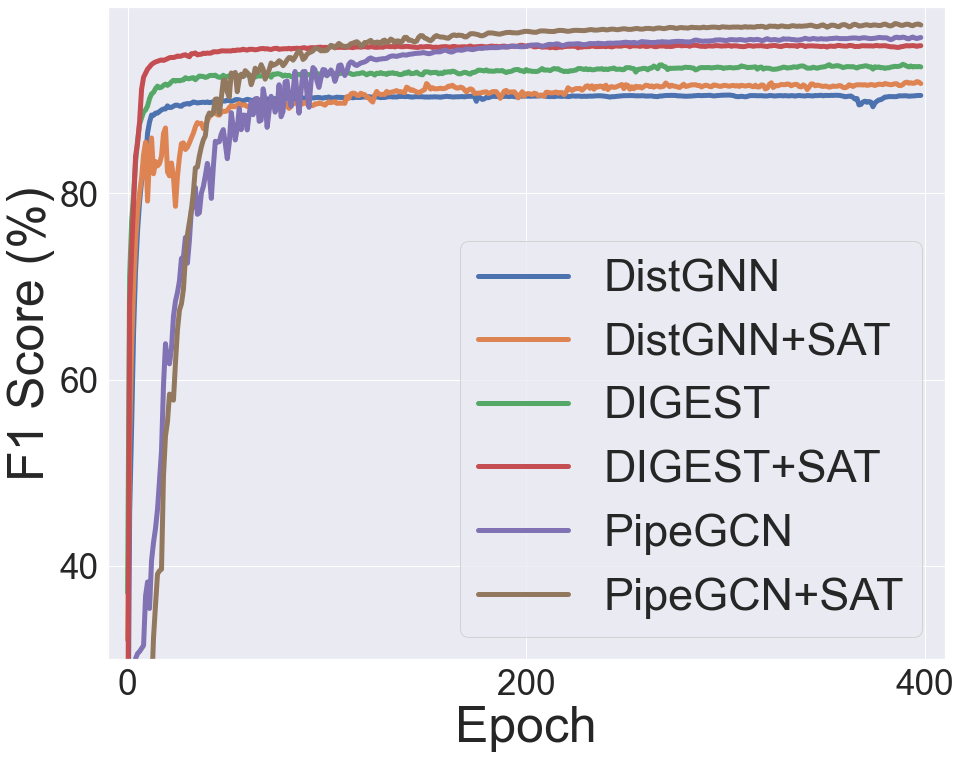}
    \caption{GCN on Reddit}
    \label{fig: gcn reddit test f1}
  \end{subfigure}
  %\hfill
  \begin{subfigure}{0.24\textwidth}
    \includegraphics[width=1.6in]{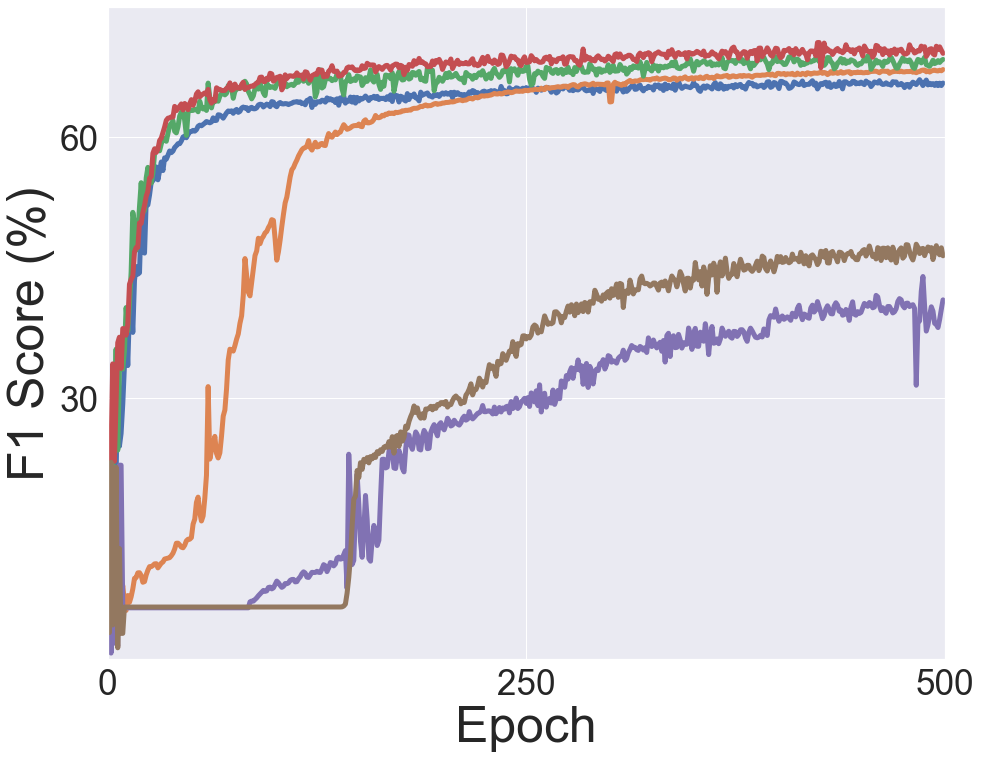}
    \caption{GCN on ogbn-arxiv}
    \label{fig: gcn arxiv test f1}
  \end{subfigure}
  %\hfill
  \begin{subfigure}{0.24\textwidth}
    \includegraphics[width=1.6in]{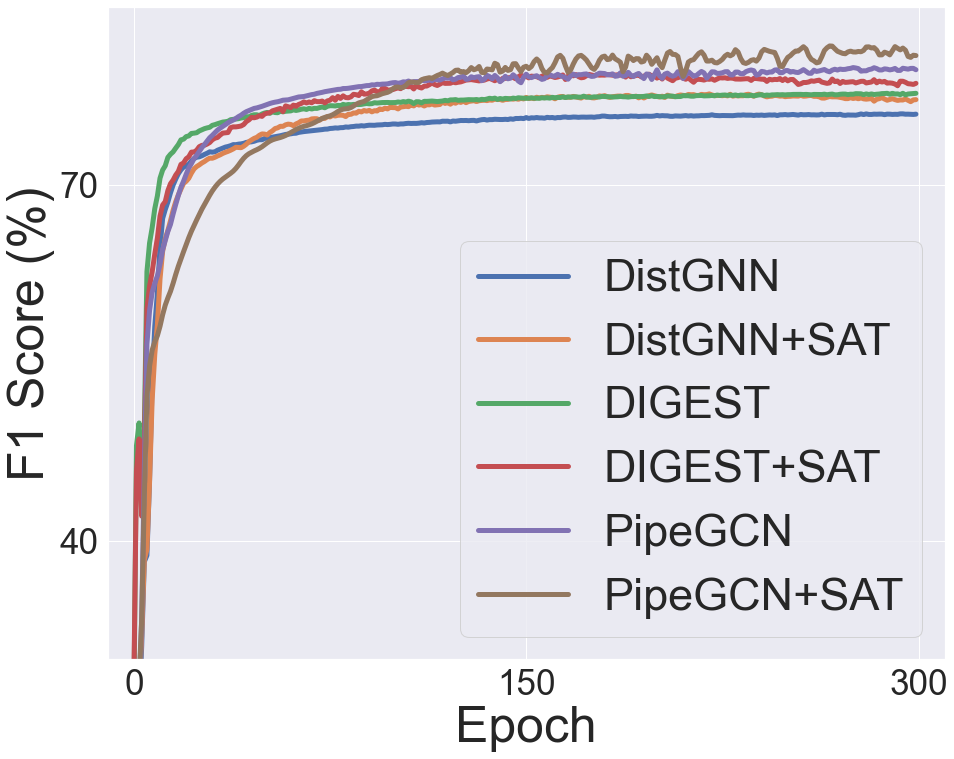}
    \caption{GCN on ogbn-products}
    \label{fig: gcn products test f1}
  \end{subfigure}
  %\hfill
  \begin{subfigure}{0.24\textwidth}
    \includegraphics[width=1.6in]{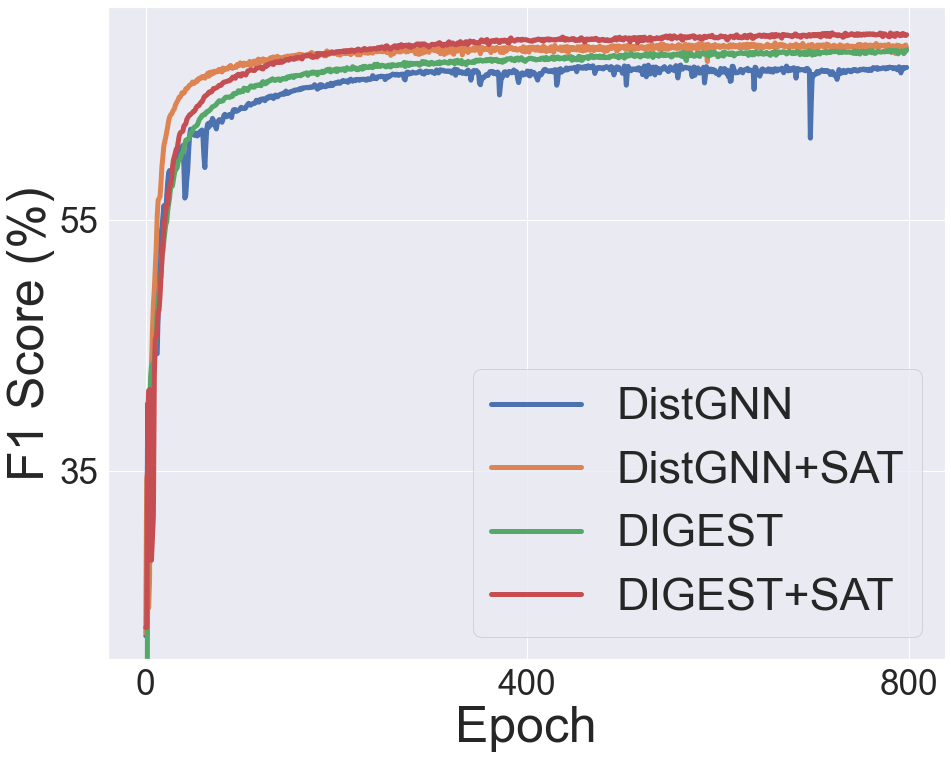}
    \caption{GAT on ogbn-arxiv}
    \label{fig: gat arxiv test f1}
  \end{subfigure}
  % \vspace{-2mm}
  \caption{\textbf{Detailed Performance Trajectories of Distributed GNN Training Methods with and without Staleness Alleviation by SAT.} These learning curves chart the evolution of global testing F1 scores across training epochs, delineating the impact of staleness alleviation on model accuracy over time. Each subplot corresponds to a different combination of GNN backbones and datasets. The ‘+’ symbol indicates the augmentation of the respective method with our SAT framework. Notably, the SAT-enhanced versions consistently reach higher F1 scores more quickly and maintain a leading performance, demonstrating the effectiveness of SAT in enhancing learning efficiency and accuracy in distributed GNN training.}
  \label{fig: distributed f1 curves}
\vspace{-3mm}
\end{figure*}

\begin{figure*}[t!]
\centering
  \begin{subfigure}{0.24\textwidth}
    \includegraphics[width=1.8in]{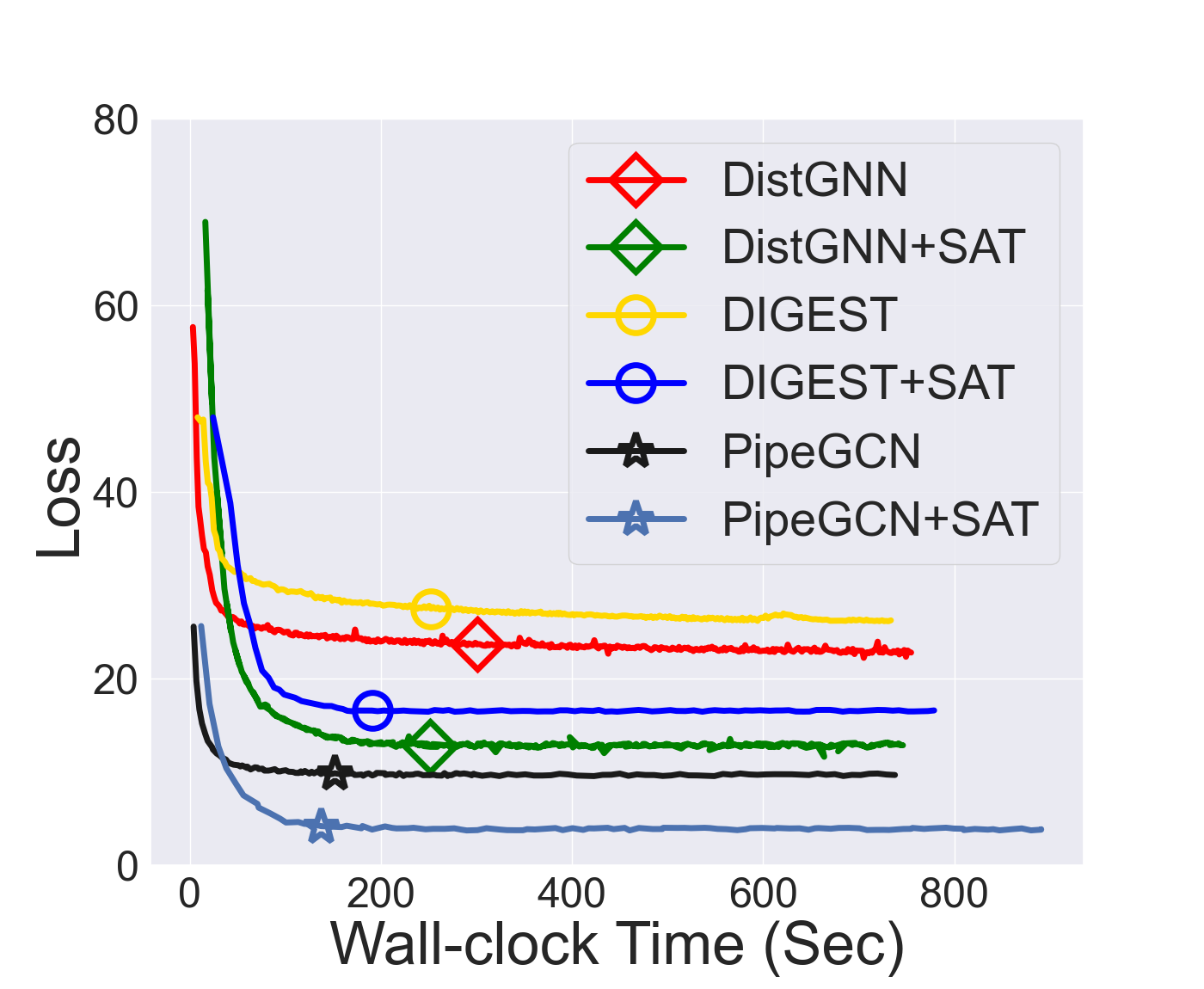}
    \caption{GCN on Reddit}
    %\label{fig:gcn reddit dist}
  \end{subfigure}
  %\hfill
  \begin{subfigure}{0.24\textwidth}
    \includegraphics[width=1.8in]{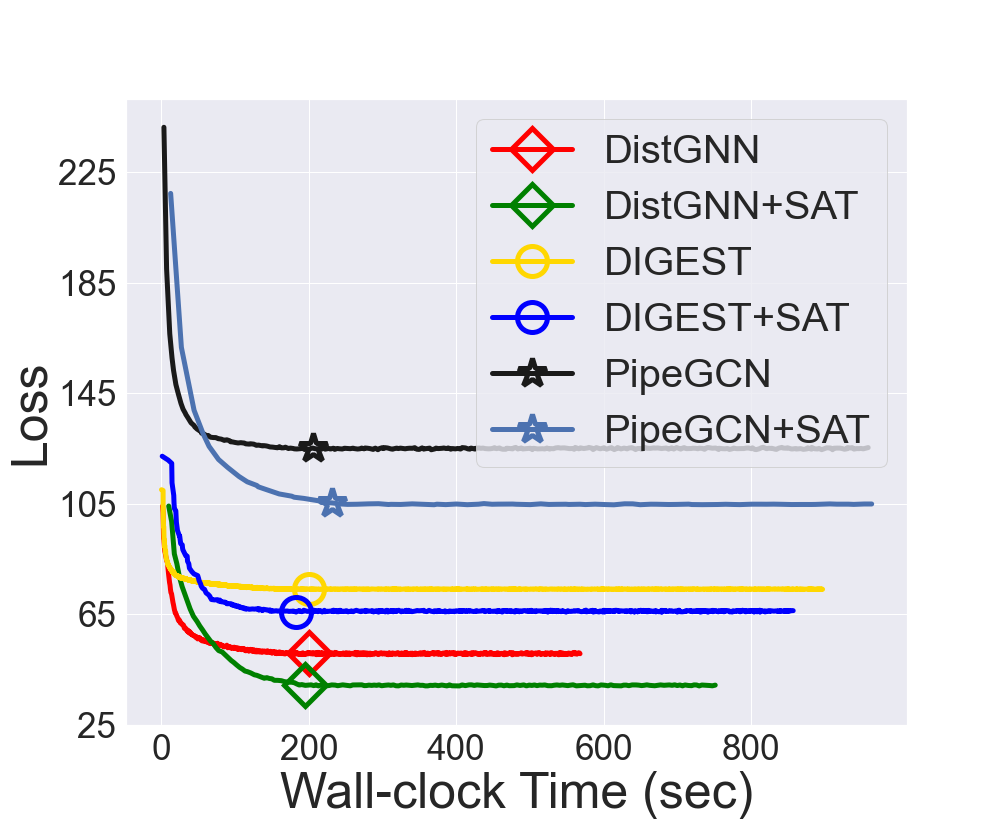}
    \caption{GCN on ogbn-arxiv}
    %\label{fig:horse racing}
  \end{subfigure}
  %\hfill
  \begin{subfigure}{0.24\textwidth}
    \includegraphics[width=1.8in]{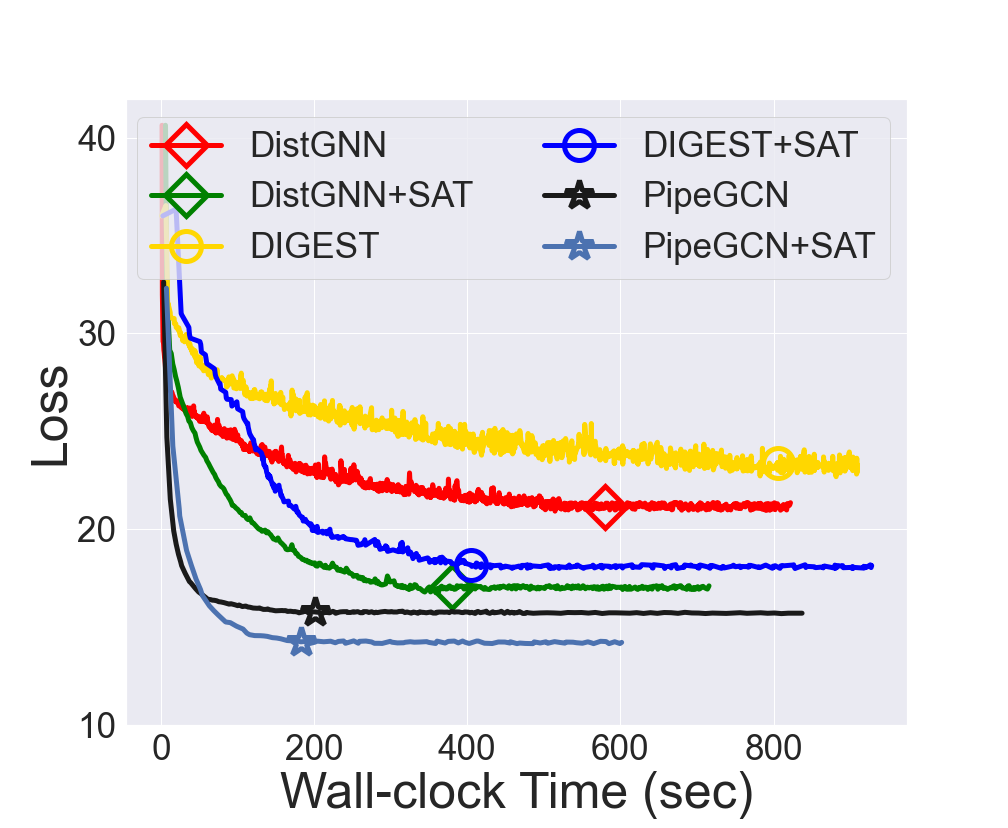}
    \caption{GCN on Flickr}
    %\label{fig:horse racing}
  \end{subfigure}
  %\hfill
  \begin{subfigure}{0.24\textwidth}
    \includegraphics[width=1.8in]{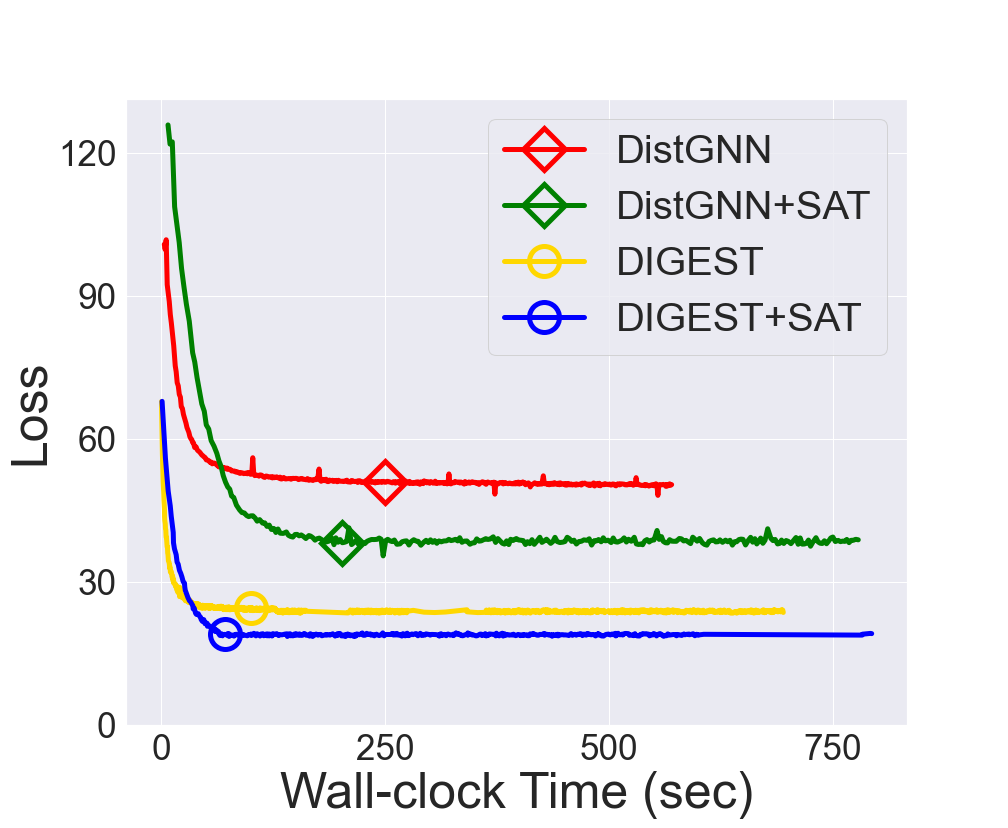}
    \caption{GAT on ogbn-arxiv}
    %\label{fig:horse racing}
  \end{subfigure}
  \vspace{-2mm}
  \caption{\textbf{Comparative Training Loss Evolution for Distributed GNNs with and without SAT.} The plots demonstrate the training loss (cross-entropy) against wall-clock time for DistGNN, DIGEST, and PipeGCN, both with (SAT) and without (vanilla) the application of Staleness-Aware Training. Convergence points are marked to highlight the improved convergence speed facilitated by SAT. In most cases, SAT enhances the training efficiency, indicated by the reduced number of epochs needed to reach convergence, thanks to the predictive correction of embedding staleness.}
  \label{fig: wallclock time curves}
\vspace{-3mm}
\end{figure*}

\noindent\textbf{Comparison Methods.} We compare our SAT with various distributed GNN training methods, including:
\begin{itemize}[noitemsep,leftmargin=*]
    \item LLCG~\cite{ramezani2021learn} proposed a framework that accelerates the sampling method under the memory budget.
    \item DistDGL~\cite{zheng2020distdgl} is the system that helps train large graphs with millions of nodes and billions of edges based on DGL.
    \item BNS-GCN~\cite{wan2022bns} proposed an efficient full-graph training of graph convolutional networks with partition-parallelism and random boundary node sampling.
    \item SANCUS~\cite{peng2022sancus} proposed a decentralized distributed GNN training framework with controlled embedding staleness.
    \item DistGNN~\cite{md2021distgnn} proposed a distributed full-graph training method based on s shared memory, a minimum vertex cut partitioning algorithm and stale embedding.
    \item DIGEST~\cite{chai2022distributed}: This is a framework that extends GNNAutoscale in a distributed synchronous manner.
    \item PipeGCN~\cite{wan2022pipegcn}: This is a framework that makes the training of large graphs more efficient through pipelined feature communication.
\end{itemize}

Note that our proposed SAT, although tailored for the distributed training setting, is general and can also be applied to reduce the staleness in traditional sampling-based GNN training approaches, including:
\begin{itemize}[noitemsep,leftmargin=*]
    \item GraphSAGE~\cite{hamilton2017inductive} is a scalable, sampling-based algorithm for learning on large graphs, enabling efficient generation of node embeddings for previously unseen data.
    \item VR-GCN~\cite{chen2018stochastic} is a framework that is trained stochastically and reduces the variance to a significant magnitude.  
    \item Cluster-GCN~\cite{chiang2019cluster} samples a block of nodes associated with a dense subgraph, identified by a graph clustering algorithm, and restricts the neighborhood search within this subgraph.
    \item GNNAutoscale (GAS)~\cite{fey2021gnnautoscale} is a scalable framework that dynamically adjusts the scale and complexity of neural networks to optimize performance and resource utilization.
\end{itemize}

\noindent\textbf{Experimental Setting Details.} For a fair comparison, we use the same optimizer (Adam), learning rate, and graph partition algorithm for all the frameworks. For parameters that are unique to e.g., PipeGCN, GNNAutoscale, DIGEST, VR-GCN, such as the number of neighbors sampled from each layer for each node and the number of layers, we keep those parameters consistent with the corresponding versions by combining with our SAT. Each of the ten frameworks has a set of parameters that are exclusively unique to that framework; for these exclusive parameters, we tune them to achieve the best performance. Please refer to the configuration files under \texttt{small\_benchmark/conf} for detailed configuration setups for all the models and datasets.

\begin{figure*}[t!]
\centering
  \begin{subfigure}{0.32\textwidth}
    \includegraphics[width=2.3in]{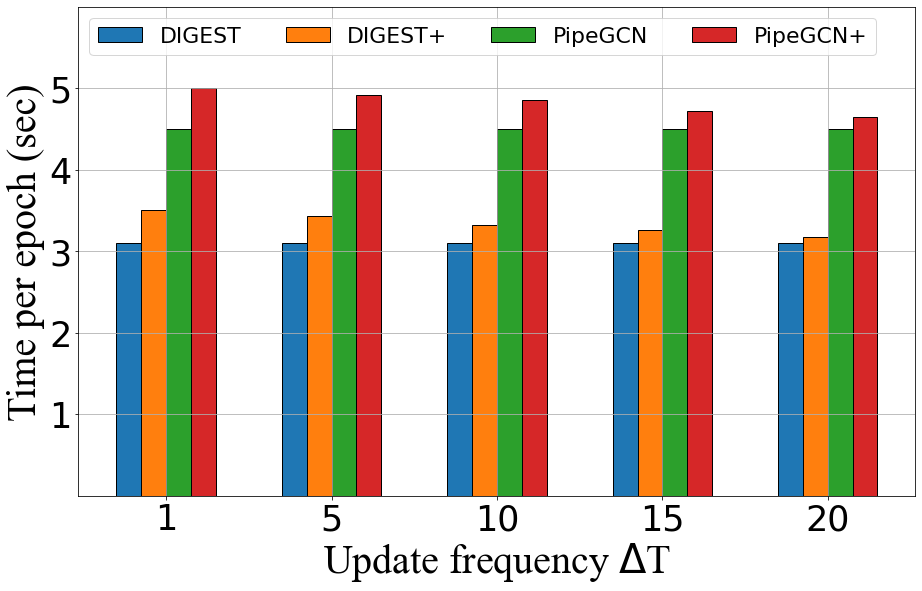}
    \caption{Training time analyses}
    \label{fig: computation overhead}
  \end{subfigure}
  %\hfill
  \begin{subfigure}{0.33\textwidth}
    \includegraphics[width=2.2in]{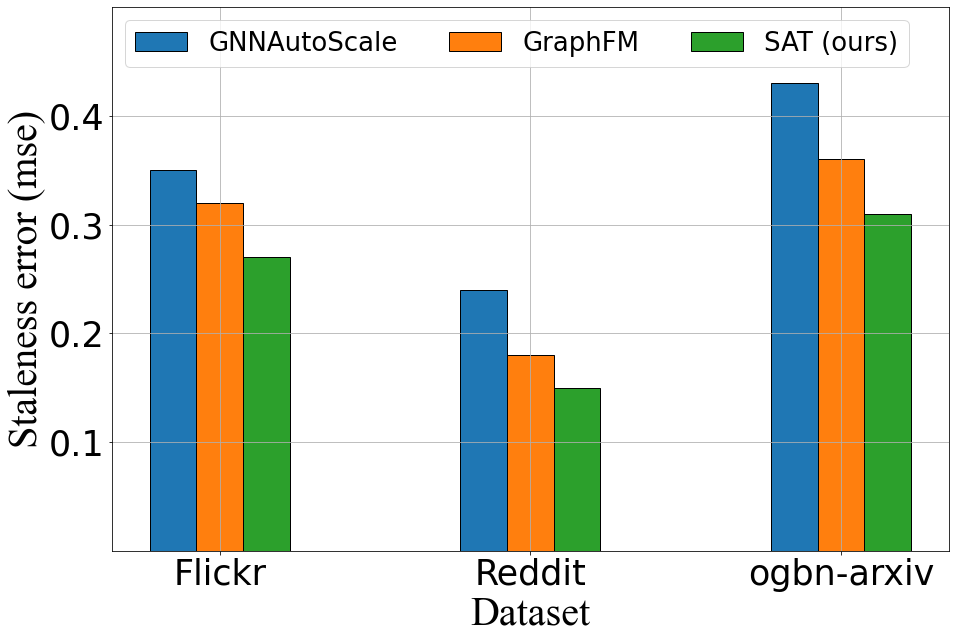}
    \caption{Effectiveness of staleness reduction}
    \label{fig: staleness reduction}
  \end{subfigure}
  %\hfill
  \begin{subfigure}{0.31\textwidth}
    \includegraphics[width=2in]{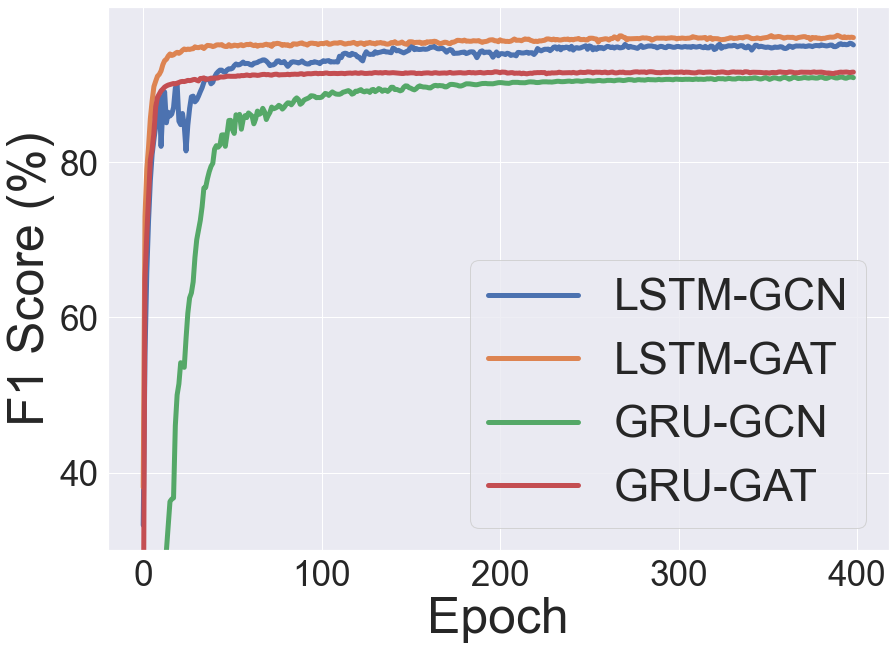}
    \caption{Embedding predictor configurations}
    \label{fig: predictor comparison}
  \end{subfigure}
  % \vspace{-5mm}
  \caption{\textbf{Left.} The training time induced by our embedding predictor is marginal and can be decreased by increasing the frequency factor $\Delta T$ for updating the predictor. $+$ denotes that our SAT is applied to this baseline for reducing embedding staleness. \textbf{Middle.} SAT can help reduce the staleness by a large margin and is more effective than other methods. \textbf{Right.} Comparison of different configurations of the embedding predictor.}
  \label{fig: additional results}
\vspace{-3mm}
\end{figure*}

\subsection{Enhanced Performance by SAT}

In this section, we offer an in-depth evaluation of the performance enhancements attributed to our Staleness Alleviation Technique (SAT) within various distributed GNN training frameworks. Specifically, we evaluate SAT against 7 distributed GNN training methods, including DistDGL, LLCG, BNS-GCN, SANCUS, DistGNN, DIGEST, and PipeGCN. Since DistGNN, DIGEST, and PipeGCN utilize historical embeddings, we implement an upgraded version for each of them by combining them with our embedding predictor.

As shown in Table~\ref{tab: distributed training results table}, it's evident that the integration of SAT with DistGNN, DIGEST, and PipeGCN has led to marked improvements in test F1 scores across all datasets. Notably, for the Reddit dataset, the SAT-augmented versions exhibit an appreciable performance leap, with the most pronounced gain observed in PipeGCN, where SAT integration has resulted in an F1 score increase from $96.20$ to $97.02$. This underscores the capability of SAT to significantly boost the accuracy of predictions in highly interconnected social network graphs. Similarly, for the ogbn-arxiv dataset, we see a substantial augmentation in performance when SAT is applied, particularly with the GCN backbone, enhancing the score from $71.54$ to $72.13$.

LLCG performs worst particularly for the Reddit dataset because in the global server correction of LLCG, only a mini-batch is trained and it is not sufficient to correct the plain GCN. This is also the reason why the authors of LLCG report the performance of a complex model with mixing GCN layers and GraphSAGE layers~\cite{ramezani2021learn}. 
DistDGL achieves good performance on some datasets (e.g., OGB-products) with uniform node sampling strategy and real-time embedding exchanging. However, frequent communication also leads to slow performance increases for dataset Flickr and poor performance for all four datasets.

The learning curves depicted in Figure~\ref{fig: distributed f1 curves} further substantiate these findings. The curves demonstrate that SAT provides a consistent and robust performance uplift across epochs. For instance, in the case of GCN on Reddit (Figure~\ref{fig: gcn reddit test f1}), the learning curve of DistGNN+SAT outpaces its non-SAT counterpart early in the training process and maintains a higher F1 score throughout, signifying not only improved performance but also a potential reduction in convergence time. Moreover, the learning curve for GCN on ogbn-products (Figure~\ref{fig: gcn products test f1}) shows that the SAT-enhanced models attain a plateau of higher F1 scores faster than those without SAT, indicating an enhanced learning efficiency.

\begin{table}[t!]
\small
\centering
% \vspace{-4mm}
\caption{\textbf{GPU memory consumption (GB) } We compare PipeGCN and DIGEST with those combined with our proposed module by including all the information inside a GCN's receptive field in a single optimization step. } 
% \vspace{-1mm}
\begin{tabular}{ccccc}
\hline
   GCN Layers & Method & Flickr & ogbn-arxiv & Reddit  \\
\hline
2-layer  & PipeGCN & 0.23 & 0.26 & 0.31 \\
\rowcolor{gray!20} 2-layer  & \emph{w/ SAT (Ours)} & 0.25 & 0.28 & 0.34 \\
2-layer  & DIGEST & 0.18 & 0.22 & 0.27 \\
\rowcolor{gray!20} 2-layer  & \emph{w/ SAT (Ours)} & 0.20 & 0.25 & 0.30 \\
\hline
3-layer  & PipeGCN & 0.26 & 0.28 & 0.35 \\
\rowcolor{gray!20} 3-layer  & \emph{w/ SAT (Ours)} & 0.28 & 0.31 & 0.37 \\
3-layer  & DIGEST & 0.20 & 0.24 & 0.29 \\
\rowcolor{gray!20} 3-layer  & \emph{w/ SAT (Ours)} & 0.23 & 0.26 & 0.31 \\
\hline
4-layer  & PipeGCN & 0.29 & 0.32 & 0.37 \\
\rowcolor{gray!20} 4-layer  & \emph{w/ SAT (Ours)} & 0.31 & 0.34 & 0.39 \\
4-layer  & DIGEST & 0.22 & 0.27 & 0.32 \\
\rowcolor{gray!20} 4-layer  & \emph{w/ SAT (Ours)} & 0.25 & 0.30 & 0.34 \\
\hline
\end{tabular}
\label{tab: memory consumption}
\vspace{-3mm}
\end{table}

\subsection{SAT Reduces More Staleness for More Graph Partitions.}

The relationship between the number of partitions used in distributed GNN training and the performance of the models is a key factor to consider when evaluating the effectiveness of different methods. As indicated by the subscripts under each dataset's name in Table~\ref{tab: distributed training results table}, which represents the number of partitions used, there is a clear trend: given a fixed method, as the number of partitions increases, the performance typically decreases. This decline can be attributed to the greater number of boundary nodes that arise with more partitions, introducing more staleness and approximation error into the model's training process. The partitioning inevitably leads to incomplete local information during the training of each partition, which in turn affects the accuracy of the embeddings.

Furthermore, examining the performance gains achieved by our proposed Staleness-Aware Training (SAT) method, it is evident that SAT is particularly effective in scenarios with a higher number of partitions. As the table shows, SAT-enhanced methods can leverage the larger "room for improvement" in these high-partition scenarios to significantly reduce embedding staleness, which is more pronounced due to the increased boundary issues and approximation errors. This is reflected in the larger performance gains observed for SAT-augmented methods in the higher-partitioned datasets. For instance, in settings with a larger number of partitions, SAT's impact on performance is more substantial, which underscores the robustness and necessity of SAT in distributed GNN frameworks where managing staleness is crucial for maintaining high-quality embeddings and, consequently, model performance.

\begin{table*}[t!]
\small
\centering
\caption{Performance comparison (test F1 score) of traditional sampling-based methods. The mean and standard deviation are calculated based on multiple runs with varying seeds. - denotes out-of-memory.  SAT can still boost the performance of sampling-based GNN training methods by reducing the staleness.}
% \vspace{-2mm}
    \begin{tabular}{lcccccccc}
    \toprule
    \multirow{2}{*}{\textbf{Model}} & \multicolumn{4}{c}{\textbf{GCN}} & \multicolumn{4}{c}{\textbf{GCN-II}}\\ \cmidrule(lr){2-5}\cmidrule(lr){6-9}
& \textbf{Flickr} & \textbf{Reddit} & \textbf{ogb-arxiv} & \textbf{ogb-products} & \textbf{Flickr} & \textbf{Reddit} & \textbf{ogb-arxiv}  & \textbf{ogb-products} \\
\midrule
GraphSAGE & 49.23\scriptsize$\pm$0.21 & 94.87\scriptsize$\pm$0.17 & 71.03\scriptsize$\pm$0.09 & 76.43\scriptsize$\pm$0.21  &  49.29\scriptsize$\pm$0.19 &  95.03\scriptsize$\pm$0.15 &  71.12\scriptsize$\pm$0.14 & 76.49\scriptsize$\pm$0.25 \\
Cluster-GCN & 47.95\scriptsize$\pm$0.31 & 94.68\scriptsize$\pm$0.25 & 70.48\scriptsize$\pm$0.15 & 76.53\scriptsize$\pm$0.35  &  48.95\scriptsize$\pm$0.34 &  94.76\scriptsize$\pm$0.22 &  71.48\scriptsize$\pm$0.14 & 76.62\scriptsize$\pm$0.26 \\
VR-GCN & 48.72\scriptsize$\pm$0.33 & 95.23\scriptsize$\pm$0.24 & - & -  &  49.72\scriptsize$\pm$0.26 & 95.76\scriptsize$\pm$0.21 & - & - \\
% \hline
\;\;\emph{w/ SAT (Ours)}    & 49.53\scriptsize$\pm$0.26 & 95.93\scriptsize$\pm$0.31 & -  &  -  & 50.53\scriptsize$\pm$0.29 &  96.32\scriptsize$\pm$0.28 &  - &  - \\
GNNAutoScale    & 53.52\scriptsize$\pm$0.76 & 95.02\scriptsize$\pm$0.41   &   71.18\scriptsize$\pm$0.27   &  76.65\scriptsize$\pm$0.23   & 54.03\scriptsize$\pm$0.45 &  96.28\scriptsize$\pm$0.47   &   73.03\scriptsize$\pm$0.25   &  77.34\scriptsize$\pm$0.21 \\
\; \emph{w/ SAT (Ours)}    & 54.03\scriptsize$\pm$0.57 & 96.03\scriptsize$\pm$0.27  & 71.51\scriptsize$\pm$0.22  & 76.83\scriptsize$\pm$0.29  & 55.21\scriptsize$\pm$0.47 & 97.03\scriptsize$\pm$0.19  & 73.62\scriptsize$\pm$0.25  & 78.02\scriptsize$\pm$0.27 \\
    \bottomrule
    \end{tabular}%
    \label{tab: comparative results non-distribute}
\end{table*}

\begin{table}[th!]
\small
\centering
% \vspace{-4mm}
\caption{Training time on ogbn-paper100M (per 100 epoch).} 
% \vspace{-1mm}
\begin{tabular}{cccc}
\hline
Model &  Time (s) & Model & Time (s)  \\
 \hline
PipeGCN  & 710 & DIGEST  & 563   \\
PipeGCN+SAT  & 769 & DIGEST+SAT  & 602   \\
\hline
Ratio  & 8.3\% & Ratio  & 6.9\%   \\
\hline
\end{tabular}
\label{tab: very large graph}
\vspace{-2mm}
\end{table}

\subsection{Convergence Speed Analyses}

The results presented in Figure~\ref{fig: wallclock time curves} provide a clear indication of the benefits introduced by the Staleness-Aware Training (SAT) approach in distributed GNN training. By integrating an embedding predictor, SAT compensates for the latency in the synchronization of distributed embeddings, effectively reducing the staleness that commonly plagues distributed training paradigms. The cross-entropy training loss curves, when viewed against wall-clock time, show that SAT not only minimizes the number of epochs required for convergence but also enhances the training loss descent trajectory in most instances.

While SAT introduces a slight increase in the computation time per epoch, this overhead is offset by a substantial decrease in the total number of epochs. The trade-off culminates in a net gain in training efficiency, as evidenced by the earlier convergence points for SAT-augmented methods. This illustrates the underlying effectiveness of SAT's predictive mechanism in maintaining the currency and relevance of embeddings, thus accelerating the training process without compromising the model's performance. Future research may explore ways to optimize the embedding predictor to further improve the trade-off between computational overhead and staleness mitigation, to advance the state-of-the-art in distributed GNN training.

\subsection{Training Time and Memory Overhead}

In-depth scrutiny of our embedding predictor reveals its influence on training time. As delineated in Figure~\ref{fig: computation overhead}, the addition of the embedding predictor contributes only a marginal increase in the training time per epoch. Significantly, this overhead demonstrates an inverse relationship with the update frequency of the embedding predictor, diminishing to a negligible level when the frequency factor, $\Delta T$, reaches 20. Through empirical analysis, we determined an optimal balance between performance enhancement and computational efficiency with $\Delta T$ set within the range of 5 to 15.

Table~\ref{tab: memory consumption} elucidates the GPU memory footprint incurred by integrating our embedding predictor with DIGEST and PipeGCN. The data exhibit a consistent pattern: the augmented frameworks incur less than a 10\% increase in memory consumption. This modest increment can be attributed to our efficient historical embedding compression technique, which conserves memory without compromising the fidelity of the embeddings. The nuances of our compression algorithm and the associated memory savings are expounded upon in the appendix, where we provide a comprehensive account of the methodology and its efficacy in a memory-constrained training environment.

Our results paint a clear picture: the SAT-equipped models, DIGEST+ and PipeGCN+, exhibit an admirable synergy of performance and efficiency. The slight memory overhead introduced by SAT is a small price to pay for the substantial gains in accuracy and training speed, making it an attractive proposition for those seeking to optimize distributed GNN training

\subsection{SAT Improves Performance for Sampling-based Methods}

We further analyze how our proposed framework SAT performs under the sampling and stochastic-training-based methods. We pick multiple widely used state-of-the-art methods, i.e., GraphSAGE, VR-GCN, Cluster-GCN, and GNNAutoScale. Since VR-GCN and GNNAutoScale utilize historical embedding, we implement an upgraded version of them by combining them with SAT. As can be seen in Table~\ref{tab: comparative results non-distribute}, both VR-GCN+ and GNNAutoScale+ outperform their counterparts in all cases, confirming the practical effectiveness of our method. 

The results, as detailed in Table IV, reveal a consistent trend: both VR-GCN+ and GNNAutoScale+ exhibit superior performance over their original versions across various datasets. This performance elevation is not only consistent but also significant, as indicated by the test F1 scores. For instance, VR-GCN+ demonstrates an improvement margin that ranges from slight in the case of the Reddit dataset to more pronounced in the ogb-arxiv dataset, as compared to the base VR-GCN.

% Last, we also compare SAT with another staleness-alleviated method GraphFM, and our result shows that SAT can consistently outperform GraphFM, which is possibly due to that, although GraphFM uses embedding momentum (a hand-crafted rule to consider cross-layer evolution) to reduce the staleness, our embedding predictor is learning-based which is more flexible and enjoys better capacity, and our temporal graph is more informative. More importantly, GraphFM is ONLY applicable to GNNAutoScale, while our SAT is a more general framework that can be applied to various GNN training frameworks.

\subsection{Measuring the Staleness of Embeddings}

The temporal relevance of historical embeddings is a critical factor in the efficacy of GNNs, particularly in a distributed setting. Figure~\ref{fig: staleness reduction} presents a comparative visualization of embedding staleness across three distinct methods. Among these, GNNAutoScale exhibits the highest level of staleness, while GraphFM and our SAT deploy strategies specifically designed to mitigate this issue.

GraphFM, a rule-based method cited from Yu et al.~\cite{yu2022graphfm}, offers staleness reduction but is constrained to applicability with GCN architectures alone. In contrast, SAT showcases its robustness with consistently the lowest staleness error across all datasets considered. This empirical evidence illustrates the superior effectiveness of SAT in maintaining the temporal accuracy of embeddings over the rule-based GraphFM approach.

The stark contrast in staleness error between SAT and the other methods underscores the advanced capability of SAT to ensure that embeddings remain current, thereby significantly enhancing the predictive performance of the models. The staleness metric, as captured in Figure~\ref{fig: staleness reduction}, serves as a testament to the sophisticated mechanism of SAT that proactively refreshes embeddings to align with the most recent graph structure and feature information. This mechanism not only underscores the superiority of SAT but also emphasizes its role as a pivotal technique in reducing the latency of information flow in distributed GNN training

\subsection{Scalability on Very-Large Graphs}

Here, we evaluate our proposed SAT framework on ogbn-paper100M, which contains more than 111 million nodes and requires multiple GPU servers for the training. We follow the same setting used by~\cite{wan2022pipegcn}, which consists of 32 GPUs. As can be seen in Table~\ref{tab: very large graph}, even on the very large graph, our proposed embedding predictor introduces roughly 7\% additional training time, which is due to our decaying fine-tuning frequency of the embedding predictor. Hence, our SAT has great scalability while boosting performance by a great margin with reduced staleness, let alone the convergence speedup.

\subsection{Comparison of Different Predictor Configurations}

Figure~\ref{fig: predictor comparison} provides a detailed comparative analysis of the performance implications of different embedding predictor architectures. The figure demonstrates that architectures based on Long Short-Term Memory (LSTM) units surpass those employing Gated Recurrent Units (GRU) in terms of F1 score. This disparity may be attributed to the inherently more complex and capable architecture of LSTMs, which can better capture and utilize the temporal dynamics of graph data.

Between the LSTM-augmented GAT and GCN models, the performance margins are narrow, indicating a competitive edge for both. However, when considering computational efficiency, LSTM-GCN emerges as the more prudent choice, striking a favorable balance between performance and resource utilization. In practical applications, LSTM-GCN's efficiency makes it a preferred model, especially when dealing with large-scale graphs or when computational resources are at a premium.

% Our findings also affirm the versatility of the proposed embedding predictor. Its general framework is designed to be model-agnostic, enabling its application across a range of more sophisticated architectures. This flexibility ensures that as GNN models evolve, the embedding predictor can continue to provide staleness alleviation and performance enhancements without being bound to a specific architecture or task.

\section{Conclusion}
\label{sec:conclusion}

Distributed GNN training with historical embeddings is a natural compromise of partition- and propagation-based methods and could enjoy the best of both worlds. However, the embedding staleness could potentially harm the model performance and convergence. In this paper, we present SAT (Staleness-Alleviated Training), a novel and scalable distributed GNN training framework that reduces the embedding staleness in a data-driven manner. We formulate the embedding prediction task as an online prediction problem over the dynamic embeddings which form a temporal graph. We provide theoretical analyses on the convergence of SAT. Extensive experiments over various comparison methods on multiple real-world graph benchmarks with different GNN operators demonstrate that our proposed SAT can greatly boost the performance of existing historical-embedding-based methods and also achieve faster convergence speed, while the additional cost is marginal.

\section{Theoretical Proof}
\label{sec: theory proof}

In this section, we provide the formal proof for our main theory presented in the paper. Specifically, we prove the convergence of SAT. First, we introduce some notions, definitions, and necessary assumptions.

\textbf{Preliminaries.} We consider GCN in our proof without loss of generality. We denote the input graph as $\mathcal{G}=(\mathcal{V},\mathcal{E})$, $L$-layer GNN as $f$, feature matrix as $X$, weight matrix as $W$. The forward propagation of one layer of  GCN is
\begin{equation}
    Z^{(\ell+1)} = PH^{(\ell)}W^{(\ell)},\quad H^{(\ell+1)} = \sigma(Z^{(\ell)})
\end{equation}
where $\ell$ is the layer index, $\sigma$ is the activation function, and $P$ is the propagation matrix following the definition of GCN~\cite {kipf2016semi}. Notice $H^{(0)}=X$. We can further define the $(\ell+1)$-th layer of  GCN as:
\begin{equation}
    f^{(\ell+1)}(H^{(\ell)},W^{(\ell)}) \coloneqq \sigma(PH^{(\ell)}W^{(\ell)})
\end{equation}

The backward propagation of GCN can be expressed as follow:
\begin{equation}
\begin{split}
    G^{(\ell)}_{H} &= \nabla_{H} f^{(\ell+1)}(H^{(\ell)},W^{(\ell)},G^{(\ell+1)}_{H}) \\
    &\coloneqq P^{\intercal}D^{(\ell+1)}(W^{(\ell+1)})^{\intercal}
\end{split}
\end{equation}
\begin{equation}
    \begin{split}
        G^{(\ell+1)}_{W} &= \nabla_{W} f^{(\ell+1)}(H^{(\ell+1)},W^{(\ell)},G^{(\ell+1)}_{H}) \\
        &\coloneqq (PH^{(\ell)})^\mathcal{\intercal}D^{(\ell+1)}
    \end{split}
\end{equation}
where
\begin{equation}
    D^{(\ell+1)} = G^{(\ell)}_{H} \circ \sigma^{\prime} (PH^{(\ell)}W^{(\ell+1)})
\end{equation}
and $\circ$ represents the Hadamard product.

Under a distributed training setting, for each subgraph $\mathcal{G}_m = (\mathcal{V}_m, \mathcal{E}_m)$, $m=1.2,\cdots,M$, the propagation matrix can be decomposed into two independent matrices, i.e. $P = P_{m,in} + P_{m,out}$, where $ P_{m,in}$ denotes the propagation matrix for nodes inside the subgraph $\mathcal{G}_m$ while $ P_{m,out}$ denotes that for neighbor nodes outside $\mathcal{G}_m$. If it will not cause confusion, we will use $P_{in}$ and $P_{out}$ in our future proof for simpler notation. 
% Our notation follows~\citep{wan2022pipegcn} in general but differs in that SAT considers that all subgraphs are computed in parallel instead of following a pipeline parallelism manner. 

For SAT, the forward propagation of a single layer of GCN can be expressed as
\begin{equation}
\begin{split}
    \tilde{Z}_{m}^{(t,\ell+1)} &= P_{in}\tilde{H}_{m}^{(t,\ell)}\tilde{W}_{m}^{(t,\ell)} + P_{out}\tilde{H}_{m}^{(t-1,\ell)}\tilde{W}_{m}^{(t,\ell)} \\
    \tilde{H}_{m}^{(t,\ell+1)} &= \sigma(\tilde{Z}_{m}^{(t,\ell)})
\end{split}
\end{equation}
where we use $\tilde{H}$ to differentiate with the counterpart without staleness, i.e., $H$ (same for other variables). $t$ is the training iteration index. Similarly, we can define each layer as a single function
\begin{equation}
\begin{split}
    &\tilde{f}_{m}^{(t,\ell+1)}(\tilde{H}_{m}^{(t,\ell)},\tilde{W}_{m}^{(t,\ell)}) \\
    &\coloneqq \sigma(P_{in}\tilde{H}_{m}^{(t,\ell)}\tilde{W}_{m}^{(t,\ell)} + P_{out}\tilde{H}_{m}^{(t-1,\ell)}\tilde{W}_{m}^{(t,\ell)})
\end{split}
\end{equation}

Note that $\tilde{H}_{m}^{(t-1,\ell-1)}$ is not part of the input since it is the stale results from the previous iteration, i.e., it can be regarded as a constant in the current iteration.

Now we can give the definition of back-propagation in SAT:
\begin{equation}
\begin{split}
    \tilde{G}^{(t,\ell)}_{H,m} &= \nabla_{H} \tilde{f}_{m}^{(t,\ell+1)}(\tilde{H}_{m}^{(\ell)},\tilde{W}^{(\ell)},\tilde{G}^{(\ell+1)}_{H,m}) \\
    &\coloneqq  P_{in}^{\intercal}\tilde{D}_{m}^{(t,\ell+1)}(\tilde{W}_{m}^{(t,\ell+1)})^{\intercal} \\
    &+ P_{out}^{\intercal}\tilde{D}_{m}^{(t-1,\ell+1)}(\tilde{W}_{m}^{(t,\ell+1)})^{\intercal}
\end{split}
\end{equation}
\begin{equation}
\begin{split}
    \tilde{G}^{(t,\ell+1)}_{W,m} &= \nabla_{W} \tilde{f}_{m}^{(t,\ell+1)}(\tilde{H}_{m}^{(t,\ell+1)},\tilde{W}_{m}^{(t,\ell)},\tilde{G}^{(t,\ell+1)}_{H,m}) \\
    &\coloneqq (P_{in}\tilde{H}_{m}^{(t,\ell)} + P_{out}\tilde{H}_{m}^{(t-1,\ell-1)})^{\intercal}\tilde{D}_{m}^{(t,\ell+1)}
\end{split}
\end{equation}
where
\begin{equation}
\begin{split}
    \tilde{D}_{m}^{(t,\ell+1)} = & G^{(\ell)}_{H,m} \circ \sigma^{\prime} (P_{in}\tilde{H}_{m}^{(t,\ell)}\tilde{W}_{m}^{(t,\ell)} \\
    &+ P_{out}\tilde{H}_{m}^{(t-1,\ell-1)}\tilde{W}_{m}^{(t,\ell)})
\end{split}
\end{equation}

In our proof, we use $\mathcal{L}(W^{(t)})$ to denote the global loss with GCN parameter $W$ after $t$ iterations, and use $\mathcal{\tilde{L}}_{m}(W_{m}^{(t)})$ to denotes the local loss for the $m$-th subgraph with model parameter $W_{m}^{(t)}$ after $t$ iterations computed by SAT.

\textbf{Assumptions.} Here we introduce some assumptions about the GCN model and the original input graph. These assumptions are standard ones that are also used in~\cite{chen2018stochastic,cong2021importance,wan2022pipegcn}.

\begin{assumption}
The loss function $Loss(\cdot,\cdot)$ is $C_{Loss}$-Lipchitz continuous and $L_{Loss}$-Lipschitz smooth with respect to the last layer's node representation, i.e.,
\begin{equation}
\begin{split}
   & \vert Loss(\mathbf{h}^{(L)}_v,\mathbf{y}_v) - Loss(\mathbf{h}^{(L)}_w,\mathbf{y}_v) \vert \\
   &\leq C_{Loss}\Vert \mathbf{h}^{(L)}_v - \mathbf{h}^{(L)}_w \Vert_2
\end{split}
\end{equation}
and
\begin{equation}
\begin{split}
    &\Vert \nabla Loss(\mathbf{h}^{(L)}_v,\mathbf{y}_v) - \nabla Loss(\mathbf{h}^{(L)}_w,\mathbf{y}_v) \Vert_2 \\&
    \leq L_{Loss}\Vert \mathbf{h}^{(L)}_v - \mathbf{h}^{(L)}_w \Vert_2
\end{split}
\end{equation}
\label{ass:lipschitz of loss}
\end{assumption}

\begin{assumption}
The activation function $\sigma(\cdot)$ is $C_{\sigma}$-Lipchitz continuous and $L_{\sigma}$-Lipschitz smooth, i.e.
\begin{equation}
    \Vert  \sigma(Z^{(\ell)}_1) - \sigma(Z^{(\ell)}_2) \Vert_2 \leq C_{\sigma}\Vert (Z^{(\ell)}_1 - Z^{(\ell)}_2 \Vert_2 
\end{equation}
and
\begin{equation}
   \Vert  \sigma^{\prime}(Z^{(\ell)}_1) - \sigma^{\prime}(Z^{(\ell)}_2) \Vert_2 \leq L_{\sigma}\Vert (Z^{(\ell)}_1 - Z^{(\ell)}_2 \Vert_2
\end{equation}
\label{ass:lipschitz of activation}
\end{assumption}

\begin{assumption}
$\forall\;\ell$ that $\ell=1,2,\cdots,L$, we have
\begin{equation}
    \Vert W^{(\ell)}\Vert_{F} \leq K_{W},\; \Vert P^{(\ell)}\Vert_{P} \leq K_{W},\; \Vert X^{(\ell)}\Vert_{F} \leq K_{X}.
\end{equation}
\end{assumption}

\begin{assumption}
Let $\hat{H}_{m}^{(t,\ell+1)}$ be the historical embedding before being corrected by the embedding predictor. The staleness satisfies the non-increasing property: 
\begin{equation}
       \Vert \tilde{H}_{m}^{(t,\ell+1)} - {H}_{m}^{(t,\ell+1)} \Vert \leq \Vert \hat{H}_{m}^{(t,\ell+1)} - {H}_{m}^{(t,\ell+1)} \Vert
\end{equation}
\end{assumption}

Now we can introduce the proof of our Theorem 4.1. We consider a GCN with $L$ layers that is $L_f$-Lipschitz smooth, i.e., $\Vert\nabla\mathcal{L}(W_1) - \nabla\mathcal{L}(W_2)\Vert_2 \leq L_f \Vert W_1 - W_2 \Vert_2$.

\begin{theorem}[Formal version of Theorem 4.1]
There exists a constant $E$ such that for any arbitrarily small constant $\epsilon > 0$, we can choose a learning rate $\eta=\frac{\sqrt{M\epsilon}}{E}$ and number of training iterations $T=(\mathcal{L}({W}^{(1)})-\mathcal{L}({W}^{*}))\frac{E}{\sqrt{M}}\epsilon^{-\frac{3}{2}}$, such that
\begin{equation}
\vspace{-2mm}
    \frac{1}{T}\sum_{t=1}^T \Vert \nabla \mathcal{L}({W}^{(t)}) \Vert^{2} \leq \mathcal{O}(\frac{1}{T^{\frac{2}{3}}M^{\frac{1}{3}}})
\end{equation}
where ${W}^{(t)}$and ${W}^{*}$ denotes the parameters at iteration $t$ and the optimal one, respectively.
\label{thm:convergence_copy}
\end{theorem}

\begin{proof}
Beginning from the assumption of smoothness of loss function, 
\begin{equation}
\begin{split}
\begin{split}
    \mathcal{L}(W^{t+1}) \leq \mathcal{L}(W^{t}) &+ \left\langle \nabla\mathcal{L}(W^{t}), W^{(t+1)} - W^{(t)} \right\rangle \\
    &+ \frac{L_f}{2}\Vert  W^{(t+1)} - W^{(t)} \Vert_{2}^2 
\end{split}
\end{split}
\end{equation}
Recall that the update rule of SAT is
\begin{equation}
    W^{(t+1)} = W^{(t)} - \frac{\eta}{M} \sum_{m=1}^M \nabla \mathcal{\tilde{L}}_{m}(W_{m}^{(t)})
\end{equation}
so we have
\begin{equation}
\begin{split}
     & \mathcal{L}(W^{t}) + \left\langle \nabla\mathcal{L}(W^{t}), W^{(t+1)} - W^{(t)} \right\rangle \\
     &+ \frac{L_f}{2}\Vert  W^{(t+1)} - W^{(t)} \Vert_{2}^2  \\
    =& \mathcal{L}(W^{t}) - \eta\left\langle \nabla\mathcal{L}(W^{t}),\frac{1}{M} \sum_{m=1}^M \nabla \mathcal{\tilde{L}}_{m}(W_{m}^{(t)}) \right\rangle \\
    &+ \frac{\eta^{2}L_f}{2}\left\Vert  \frac{1}{M} \sum_{m=1}^M \nabla \mathcal{\tilde{L}}_{m}(W_{m}^{(t)}) \right\Vert_{2}^2
\end{split}
\end{equation}

Denote $\delta^{(t)}_{m} = \nabla\mathcal{\tilde{L}}_{m}(W_{m}^{(t)}) - \nabla\mathcal{L}_{m}(W_{m}^{(t)})$, we have
\begin{equation}
\begin{split}
&\mathcal{L}(W^{t+1}) \leq  \mathcal{L}(W^{t}) - \\
&\eta\left\langle \nabla\mathcal{L}(W^{t}),\frac{1}{M} \sum_{m=1}^M \left(\nabla\mathcal{L}_{m}(W_{m}^{(t)}) + \delta^{(t)}_{m}\right) \right\rangle \\
&+ \frac{\eta^{2}L_f}{2} \left\Vert\frac{1}{M} \sum_{m=1}^M \left(\nabla\mathcal{L}_{m}(W_{m}^{(t)}) + \delta^{(t)}_{m}\right) \right\Vert_{2}^2 
\end{split}
\label{eq:tmp}
\end{equation}
Without loss of generality, assume the original graph can be divided evenly into M subgraphs and denote $N=\vert\mathcal{V}\vert$ as the original graph size, i.e., $N=M\cdot S$, where $S$ is each subgraph size. Notice that
\begin{equation}
\begin{split}
    \nabla\mathcal{L}(W^{t}) &= \frac{1}{N}\sum_{i=1}^{N}\nabla Loss(f_{i}^{(L)},y_i) \\
    &= \frac{1}{M}\Big\{ \sum_{m=1}^{M} \frac{1}{S}\sum_{i=1}^{S} \nabla Loss(f_{m,i}^{(L)},y_{m,i}) \Big\}
\end{split}
\end{equation}
which is essentially 
\begin{equation}
    \nabla\mathcal{L}(W^{t}) = \frac{1}{M}\sum_{m=1}^{M} \nabla\mathcal{L}_{m}(W_{m}^{(t)})
\end{equation}

Plugging the equation above into Eq.~\ref{eq:tmp}, we have
\begin{equation*}
    \mathcal{L}(W^{t+1}) \leq \mathcal{L}(W^{t}) - \frac{\eta}{2} \Vert \nabla\mathcal{L}(W^{t}) \Vert_{2}^{2} + \frac{\eta^{2}L_f}{2} \Big\Vert \frac{1}{M}\sum_{m=1}^{M}\delta_{m}^{(t)} \Big\Vert_{2}^{2} 
\end{equation*}
which after rearranging the terms leads to
\begin{equation*}
    \Vert \nabla\mathcal{L}(W^{t}) \Vert_{2}^{2} \leq \frac{2}{\eta}(\mathcal{L}(W^{t}) - \mathcal{L}(W^{t+1})) + \eta L_f \Big\Vert \frac{1}{M}\sum_{m=1}^{M}\delta_{m}^{(t)} \Big\Vert_{2}^{2} 
\end{equation*}
By taking $\eta < 1/L_f$, using the four assumptions defined earlier and Corollary A.10 in~\cite{wan2022pipegcn}, and summing up the inequality above over all iterations, i.e., $t=1,2,\cdots,T$, we have
\begin{equation*}
\begin{split}
    \frac{1}{T}\sum_{t=1}^T \Vert \nabla \mathcal{L}({W}^{(t)}) \Vert^{2} &\leq \frac{2}{\eta T}\big(\mathcal{L}(W^{1}) - \mathcal{L}(W^{T+1})\big) + \frac{\eta^{2}E^{2}}{M} \\
    &\leq \frac{2}{\eta T}\big(\mathcal{L}(W^{1}) - \mathcal{L}(W^{*})\big) + \frac{\eta^{2}E^{2}}{M}
\end{split}
\end{equation*}
where $W^{*}$ denotes the minima of the loss function and $E$ is a constant depends on $E^{\prime}$.

Finally, taking $\eta = \frac{\sqrt{M\epsilon}}{E}$ and $T=(\mathcal{L}({W}^{(1)})-\mathcal{L}({W}^{*}))\frac{E}{\sqrt{M}}\epsilon^{-\frac{3}{2}}$ finishes the proof.

\end{proof}

\bibliographystyle{IEEEtran} 
\bibliography{sat}

\begin{thebibliography}{10}
\providecommand{\url}[1]{#1}
\csname url@samestyle\endcsname
\providecommand{\newblock}{\relax}
\providecommand{\bibinfo}[2]{#2}
\providecommand{\BIBentrySTDinterwordspacing}{\spaceskip=0pt\relax}
\providecommand{\BIBentryALTinterwordstretchfactor}{4}
\providecommand{\BIBentryALTinterwordspacing}{\spaceskip=\fontdimen2\font plus
\BIBentryALTinterwordstretchfactor\fontdimen3\font minus \fontdimen4\font\relax}
\providecommand{\BIBforeignlanguage}[2]{{%
\expandafter\ifx\csname l@#1\endcsname\relax
\typeout{** WARNING: IEEEtran.bst: No hyphenation pattern has been}%
\typeout{** loaded for the language `#1'. Using the pattern for}%
\typeout{** the default language instead.}%
\else
\language=\csname l@#1\endcsname
\fi
#2}}
\providecommand{\BIBdecl}{\relax}
\BIBdecl

\bibitem{dai2016discriminative}
H.~Dai, B.~Dai, and L.~Song, ``Discriminative embeddings of latent variable models for structured data,'' in \emph{International conference on machine learning}.\hskip 1em plus 0.5em minus 0.4em\relax PMLR, 2016, pp. 2702--2711.

\bibitem{ying2018graph}
R.~Ying, R.~He, K.~Chen, P.~Eksombatchai, W.~L. Hamilton, and J.~Leskovec, ``Graph convolutional neural networks for web-scale recommender systems,'' in \emph{Proceedings of the 24th ACM SIGKDD international conference on knowledge discovery \& data mining}, 2018, pp. 974--983.

\bibitem{lei2019gcn}
K.~Lei, M.~Qin, B.~Bai, G.~Zhang, and M.~Yang, ``Gcn-gan: A non-linear temporal link prediction model for weighted dynamic networks,'' in \emph{IEEE INFOCOM 2019-IEEE Conference on Computer Communications}.\hskip 1em plus 0.5em minus 0.4em\relax IEEE, 2019, pp. 388--396.

\bibitem{dorylus_osdi21}
\BIBentryALTinterwordspacing
J.~Thorpe, Y.~Qiao, J.~Eyolfson, S.~Teng, G.~Hu, Z.~Jia, J.~Wei, K.~Vora, R.~Netravali, M.~Kim, and G.~H. Xu, ``Dorylus: Affordable, scalable, and accurate {GNN} training with distributed {CPU} servers and serverless threads,'' in \emph{15th USENIX Symposium on Operating Systems Design and Implementation (OSDI 21)}.\hskip 1em plus 0.5em minus 0.4em\relax USENIX Association, Jul. 2021, pp. 495--514. [Online]. Available: \url{https://www.usenix.org/conference/osdi21/presentation/thorpe}
\BIBentrySTDinterwordspacing

\bibitem{ramezani2021learn}
M.~Ramezani, W.~Cong, M.~Mahdavi, M.~T. Kandemir, and A.~Sivasubramaniam, ``Learn locally, correct globally: A distributed algorithm for training graph neural networks,'' \emph{arXiv preprint arXiv:2111.08202}, 2021.

\bibitem{wan2022pipegcn}
C.~Wan, Y.~Li, C.~R. Wolfe, A.~Kyrillidis, N.~S. Kim, and Y.~Lin, ``Pipegcn: Efficient full-graph training of graph convolutional networks with pipelined feature communication,'' \emph{arXiv preprint arXiv:2203.10428}, 2022.

\bibitem{chai2022distributed}
Z.~Chai, G.~Bai, L.~Zhao, and Y.~Cheng, ``Distributed graph neural network training with periodic historical embedding synchronization,'' \emph{arXiv preprint arXiv:2206.00057}, 2022.

\bibitem{angerd2020distributed}
A.~Angerd, K.~Balasubramanian, and M.~Annavaram, ``Distributed training of graph convolutional networks using subgraph approximation,'' \emph{arXiv preprint arXiv:2012.04930}, 2020.

\bibitem{jia2020improving}
Z.~Jia, S.~Lin, M.~Gao, M.~Zaharia, and A.~Aiken, ``Improving the accuracy, scalability, and performance of graph neural networks with roc,'' \emph{Proceedings of Machine Learning and Systems}, vol.~2, pp. 187--198, 2020.

\bibitem{ma2019neugraph}
L.~Ma, Z.~Yang, Y.~Miao, J.~Xue, M.~Wu, L.~Zhou, and Y.~Dai, ``$\{$NeuGraph$\}$: Parallel deep neural network computation on large graphs,'' in \emph{2019 USENIX Annual Technical Conference (USENIX ATC 19)}, 2019, pp. 443--458.

\bibitem{zhu2019aligraph}
R.~Zhu, K.~Zhao, H.~Yang, W.~Lin, C.~Zhou, B.~Ai, Y.~Li, and J.~Zhou, ``Aligraph: a comprehensive graph neural network platform,'' \emph{arXiv preprint arXiv:1902.08730}, 2019.

\bibitem{zheng2020distdgl}
D.~Zheng, C.~Ma, M.~Wang, J.~Zhou, Q.~Su, X.~Song, Q.~Gan, Z.~Zhang, and G.~Karypis, ``Distdgl: distributed graph neural network training for billion-scale graphs,'' in \emph{2020 IEEE/ACM 10th Workshop on Irregular Applications: Architectures and Algorithms (IA3)}.\hskip 1em plus 0.5em minus 0.4em\relax IEEE, 2020, pp. 36--44.

\bibitem{tripathy2020reducing}
A.~Tripathy, K.~Yelick, and A.~Bulu{\c{c}}, ``Reducing communication in graph neural network training,'' in \emph{SC20: International Conference for High Performance Computing, Networking, Storage and Analysis}.\hskip 1em plus 0.5em minus 0.4em\relax IEEE, 2020, pp. 1--14.

\bibitem{fey2021gnnautoscale}
M.~Fey, J.~E. Lenssen, F.~Weichert, and J.~Leskovec, ``Gnnautoscale: Scalable and expressive graph neural networks via historical embeddings,'' in \emph{International Conference on Machine Learning}.\hskip 1em plus 0.5em minus 0.4em\relax PMLR, 2021, pp. 3294--3304.

\bibitem{yu2022graphfm}
H.~Yu, L.~Wang, B.~Wang, M.~Liu, T.~Yang, and S.~Ji, ``Graphfm: Improving large-scale gnn training via feature momentum,'' in \emph{International Conference on Machine Learning}.\hskip 1em plus 0.5em minus 0.4em\relax PMLR, 2022, pp. 25\,684--25\,701.

\bibitem{yang2019aligraph}
H.~Yang, ``Aligraph: A comprehensive graph neural network platform,'' in \emph{Proceedings of the 25th ACM SIGKDD international conference on knowledge discovery \& data mining}, 2019, pp. 3165--3166.

\bibitem{wang2019deep}
M.~Wang, D.~Zheng, Z.~Ye, Q.~Gan, M.~Li, X.~Song, J.~Zhou, C.~Ma, L.~Yu, Y.~Gai \emph{et~al.}, ``Deep graph library: A graph-centric, highly-performant package for graph neural networks,'' \emph{arXiv preprint arXiv:1909.01315}, 2019.

\bibitem{gandhi2021p3}
S.~Gandhi and A.~P. Iyer, ``P3: Distributed deep graph learning at scale,'' in \emph{15th $\{$USENIX$\}$ Symposium on Operating Systems Design and Implementation ($\{$OSDI$\}$ 21)}, 2021, pp. 551--568.

\bibitem{shao2022distributed}
Y.~Shao, H.~Li, X.~Gu, H.~Yin, Y.~Li, X.~Miao, W.~Zhang, B.~Cui, and L.~Chen, ``Distributed graph neural network training: A survey,'' \emph{arXiv preprint arXiv:2211.00216}, 2022.

\bibitem{huo2018decoupled}
Z.~Huo, B.~Gu, H.~Huang \emph{et~al.}, ``Decoupled parallel backpropagation with convergence guarantee,'' in \emph{International Conference on Machine Learning}.\hskip 1em plus 0.5em minus 0.4em\relax PMLR, 2018, pp. 2098--2106.

\bibitem{xu2020acceleration}
A.~Xu, Z.~Huo, and H.~Huang, ``On the acceleration of deep learning model parallelism with staleness,'' in \emph{Proceedings of the IEEE/CVF Conference on Computer Vision and Pattern Recognition}, 2020, pp. 2088--2097.

\bibitem{chen2018stochastic}
J.~Chen, J.~Zhu, and L.~Song, ``Stochastic training of graph convolutional networks with variance reduction,'' in \emph{International Conference on Machine Learning}.\hskip 1em plus 0.5em minus 0.4em\relax PMLR, 2018, pp. 942--950.

\bibitem{kipf2016semi}
T.~N. Kipf and M.~Welling, ``Semi-supervised classification with graph convolutional networks,'' \emph{arXiv preprint arXiv:1609.02907}, 2016.

\bibitem{zhang2021survey}
Y.~Zhang and Q.~Yang, ``A survey on multi-task learning,'' \emph{IEEE Transactions on Knowledge and Data Engineering}, vol.~34, no.~12, pp. 5586--5609, 2021.

\bibitem{wu2022graph}
L.~Wu, P.~Cui, J.~Pei, L.~Zhao, and L.~Song, ``Graph neural networks,'' in \emph{Graph Neural Networks: Foundations, Frontiers, and Applications}.\hskip 1em plus 0.5em minus 0.4em\relax Springer, 2022, pp. 27--37.

\bibitem{rajeswaran2019meta}
A.~Rajeswaran, C.~Finn, S.~Kakade, and S.~Levine, ``Meta-learning with implicit gradients,'' 2019.

\bibitem{hu2020open}
W.~Hu, M.~Fey, M.~Zitnik, Y.~Dong, H.~Ren, B.~Liu, M.~Catasta, and J.~Leskovec, ``Open graph benchmark: Datasets for machine learning on graphs,'' \emph{Advances in neural information processing systems}, vol.~33, pp. 22\,118--22\,133, 2020.

\bibitem{zeng2019graphsaint}
H.~Zeng, H.~Zhou, A.~Srivastava, R.~Kannan, and V.~Prasanna, ``Graphsaint: Graph sampling based inductive learning method,'' \emph{arXiv preprint arXiv:1907.04931}, 2019.

\bibitem{wan2022bns}
C.~Wan, Y.~Li, A.~Li, N.~S. Kim, and Y.~Lin, ``Bns-gcn: Efficient full-graph training of graph convolutional networks with partition-parallelism and random boundary node sampling,'' \emph{Proceedings of Machine Learning and Systems}, vol.~4, pp. 673--693, 2022.

\bibitem{peng2022sancus}
J.~Peng, Z.~Chen, Y.~Shao, Y.~Shen, L.~Chen, and J.~Cao, ``Sancus: sta le n ess-aware c omm u nication-avoiding full-graph decentralized training in large-scale graph neural networks,'' \emph{Proceedings of the VLDB Endowment}, vol.~15, no.~9, pp. 1937--1950, 2022.

\bibitem{md2021distgnn}
V.~Md, S.~Misra, G.~Ma, R.~Mohanty, E.~Georganas, A.~Heinecke, D.~Kalamkar, N.~K. Ahmed, and S.~Avancha, ``Distgnn: Scalable distributed training for large-scale graph neural networks,'' in \emph{Proceedings of the International Conference for High Performance Computing, Networking, Storage and Analysis}, 2021, pp. 1--14.

\bibitem{hamilton2017inductive}
W.~Hamilton, Z.~Ying, and J.~Leskovec, ``Inductive representation learning on large graphs,'' \emph{Advances in neural information processing systems}, vol.~30, 2017.

\bibitem{chiang2019cluster}
W.-L. Chiang, X.~Liu, S.~Si, Y.~Li, S.~Bengio, and C.-J. Hsieh, ``Cluster-gcn: An efficient algorithm for training deep and large graph convolutional networks,'' in \emph{Proceedings of the 25th ACM SIGKDD International Conference on Knowledge Discovery \& Data Mining}, 2019, pp. 257--266.

\bibitem{cong2021importance}
W.~Cong, M.~Ramezani, and M.~Mahdavi, ``On the importance of sampling in learning graph convolutional networks,'' \emph{arXiv preprint arXiv:2103.02696}, 2021.

\end{thebibliography}

\end{document}